\documentclass[runningheads,a4paper]{llncs}

\usepackage{booktabs} 
\usepackage{microtype}
\usepackage{xspace}

\usepackage{array}
\usepackage{dsfont}
\usepackage{pifont}
\usepackage{times}
\usepackage{helvet}
\usepackage{courier}
\usepackage{color}
\usepackage{colortbl}
\usepackage{mathrsfs}
\usepackage{amscd}
\usepackage{stmaryrd}
\usepackage{amssymb}
\usepackage{amsmath}
\usepackage{graphicx}
\usepackage{tikz}
\usepackage{algorithm}
\usepackage[noend]{algpseudocode}
\algrenewcommand\algorithmicindent{0.750em}%
\usepackage{subfig}
\usepackage{multirow}

\newcommand{\commentout}[1]{}

\newcommand{\labels}{\mathcal{L}}
\newcommand{\networks}{\mathcal{N}}

\newcommand{\testsuites}{\mathcal{T}}
\newcommand{\neuronpairs}{\mathcal{O}}
\newcommand{\distance}[2]{ ||#1||_{#2}}
\newcommand{\valuefunction}{g}

\newcommand{\neuronpair}{\alpha}
\newcommand{\covered}[3]{{#1}^{#2}#3}

\newcommand{\metric}{M}

\newcommand{\setoffeatures}{\Psi}
\newcommand{\feature}{\psi}

\newcommand\scalemath[2]{\scalebox{#1}{\mbox{\ensuremath{\displaystyle #2}}}}

\newcommand\true{\textsf{true}\xspace}
\newcommand\false{\textsf{false}\xspace}

\begin{document}

\title{Testing Deep Neural Networks}

\author{Youcheng Sun\inst{1} \and Xiaowei Huang\inst{2} \and Daniel Kroening\inst{1} \and James Sharp\inst{3} \and Matthew Hill\inst{3} \and Rob Ashmore\inst{3}}

\institute{Department of Computer Science, University of Oxford, UK
\and
Department of Computer Science, University of Liverpool, UK
\and
Defence Science and Technology Laboratory (Dstl), UK}

\authorrunning{Sun, Huang, Kroening, Sharp, Hill, Ashmore}
\maketitle



%


\begin{abstract}
Deep neural networks (DNNs) have a wide range of applications, and software
employing them must be thoroughly tested, especially in safety-critical
domains.  However, traditional software test coverage metrics cannot be
applied directly to DNNs.  In this paper, inspired by the MC/DC coverage
criterion, we propose a family of four novel test criteria that are tailored to
structural features of DNNs and their semantics.  We validate the criteria
by demonstrating that the generated test inputs guided via our proposed coverage
criteria are able to capture undesired behaviours in a DNN. Test cases
are generated using a symbolic approach and a gradient-based heuristic search.
By comparing them with existing methods, we show that our criteria achieve a 
balance between their ability to find bugs (proxied using adversarial examples) 
and the computational cost of test case generation.
Our experiments are conducted on state-of-the-art DNNs obtained
using popular open source datasets, including MNIST, CIFAR-10 and ImageNet.
\end{abstract}


\section{Introduction}

Artificial intelligence (AI), specifically deep neural networks (DNNs), can deliver 
human-level results in some specialist tasks. There is now a prospect of
a wide-scale deployment of DNNs in safety-critical applications 
such as self-driving cars.  This naturally raises the question how software implementing this
technology should be tested, validated and ultimately certified to meet the
requirements of the relevant safety standards \cite{huang2018safety}.

Research and industrial communities worldwide are  taking significant efforts towards the best practice for 
the safety assurance for learning-enabled autonomous systems. Among all efforts, we mention a few, including  
a proposal under consideration by IEEE to form an official technical committee
for verification  of autonomous  systems~\cite{vaswg}, the Assuring Autonomy International Programme \cite{aaip} which  investigates the certification of learned models, etc. Moreover, as stated in \cite{scsc}, the machine learning algorithm should be verified with an appropriate
level of coverage. This paper develops a technical solution to support these efforts.

The industry relies on testing as a primary means to provide
stakeholders with information about the quality of the software product or
service~\cite{kaner2006}.
Research in software testing has resulted in a broad range of approaches
to assess different software criticality levels 
(comprehensive reviews are given in e.g.,~\cite{ZHM1997,JH2011,SWMPHS2017}).  
In white-box testing, the structure of a program is exploited to
(perhaps automatically) generate test cases. 
Code coverage criteria (or metrics) have been designed to guide the
generation of test cases and evaluate the completeness of a test suite. 
For example, a test suite with 100\% statement coverage exercises all statements at
least once.  While it is arguable the extent to which coverage  ensures correct
functionality, high coverage is able to increase users' confidence (or
trust) in the program~\cite{ZHM1997}.
Structural coverage metrics are used
as a means of assessment in several high-tier safety standards, which establish
both statement and modified condition/decision coverage (MC/DC) are
applicable measures.  MC/DC was developed by NASA and has been widely
adopted.  It is used in avionics software development guidance to ensure
adequate testing of applications with the highest criticality~\cite{do178c}.

AI systems that use DNNs are typically implemented in software.  However,
(white-box) testing for traditional software cannot be directly applied to
DNNs.  In particular, the flow of control in DNNs is not sufficient to
represent the knowledge that is learned during the training phase and thus
it is not obvious how to define structural coverage criteria for
DNNs~\cite{AL2017}.  Meanwhile, DNNs exhibit different "bugs" from
traditional software.  Notably, \emph{adversarial
examples}~\cite{szegedy2014intriguing}, in which two
apparently indistinguishable inputs cause contradicted decisions, are one of the
most prominent safety concerns in DNNs.

We believe that the testing of DNNs, guided by proper coverage criteria, must help
developers find bugs, quantify network robustness and  analyse its internal structures.
Also, developers can use the generated adversarial examples to re-train and improve
the network. These enable developers to understand and compare different networks
for any safety related argument.  

Technically, DNNs
contain not only an architecture, which bears some similarity with
traditional software programs, but also a large set of parameters, which are
tuned by the training procedure.  Any approach to testing DNNs needs to
consider the unique properties of DNNs, such as the syntactic connections
between neurons in adjacent layers (neurons in a given layer interact with
each other and then pass information to higher layers), the ReLU (Rectified
Linear Unit) activation functions and the semantic relationship between
layers. 

In this paper, we propose a novel, white-box testing methodology for DNNs.
%
In particular, we propose a family of four test criteria, inspired by the MC/DC test
criterion~\cite{HVCR2001} from traditional software testing, that {fit the
distinct properties of DNNs mentioned above.} 
It is known that an overly weak criterion may lead to insufficient testing,
e.g., 100\% neuron coverage~\cite{PCYJ2017} can be
achieved by a simple test suite comprised of few input vectors from the
training dataset, and an overly strong criterion may lead to computational
intractability, e.g., 100\% safety coverage is shown in \cite{WHK2018} as
difficult to achieve (NP-hard). 
Our criteria, when applied to guide test case generation, can achieve both
intensive testing (i.e., non-trivial to achieve 100\% coverage) and
computational feasibility. 
As a matter of fact, excepting the safety coverage criterion in \cite{WHK2018}, all existing structural 
test coverage criteria for DNNs~\cite{PCYJ2017,ma2018deepgauge} are special cases of our proposed criteria.
Our criteria are the first work that is able to capture and quantify causal relations existing in a DNN that are
critical for understanding the neural network behaviour~\cite{YCNFL2015,olah2018the}.

\commentout{
There exist mainly two coverage criteria for DNNs: neuron coverage~\cite{PCYJ2017}
(and its generalisations \cite{ma2018deepgauge}) and
safety coverage~\cite{WHK2018}, both of which have been proposed recently.  Our
experiments show that neuron coverage is too coarse: 100\% coverage can be
achieved by a simple test suite comprised of few input vectors from the training dataset.  On
the other hand, safety coverage is black-box, too fine, and it is
computationally too expensive to compute a test suite in reasonable time.
Moreover, our four proposed criteria are incomparable with each other, and
complement each other in capturing different kinds of causal relationship between DNN features and
in guiding the generation of test cases.
}

Subsequently, we validate the utility of our MC/DC variant by applying it to different approaches 
to DNN testing. At first, we adopt state-of-the-art concolic testing for 
DNNs~\cite{sun2018concolic}.
Concolic testing combines concrete testing and symbolic encoding of DNNs. 
Specifically, the linear programming (LP) based algorithm produces a new test case (i.e., an input vector) by 
encoding a fragment of the DNN 
and then optimises
over an objective that is to minimise the difference between the new and the
current input vector.  LP can be solved efficiently in PTIME, so the concolic
test case generation algorithms can generate a test suite with low
computational cost for small to medium-sized DNNs.  Meanwhile, we develop a gradient descent (GD)
based algorithm that takes the test condition as the optimisation objective and
searches for satisfiable test cases in an adaptive manner under the guidance
of the first-order derivative of the DNNs, which is able to work with
large-scale DNNs.
%

Finally, 
%
we experiment with our test coverage criteria on state-of-the-art neural networks of different sizes 
(from a few hundred up to millions of neurons) to demonstrate their utility with respect to
four aspects: 
\emph{\ding{192}~bug finding
\ding{193}~DNN safety statistics \ding{194}~testing efficiency \ding{195}~DNN internal structure
analysis}. Bugs here refer to adversarial examples.
  


\section{Preliminaries: Deep Neural Networks}
\label{sec:dnn}

\newcommand{\real}{\mathds{R}}

A (feedforward and deep) neural network, or DNN, is a tuple $\networks=(L,
T, \Phi)$, where $L=\{L_k~|~k\in\{1..K\}\}$ is a set of layers,
$T\subseteq L\times L$ is a set of connections between layers and 
$\Phi=\{\phi_k~|~k\in\{2,\ldots,K\}\}$ is a set of functions, one for each non-input layer.
%
%
In a DNN, $L_1$ is the \emph{input} layer, $L_{K}$ is the \emph{output} layer
and layers other than input and output layers are called \emph{hidden layers}.
Each layer $L_k$ consists of $s_k$ 
\emph{neurons} (or nodes).
The $l$-th node of layer $k$ is denoted by $n_{k,l}$.
Each node $n_{k,l}$ for $1 < k< K$ and  $1\leq l\leq s_k$ is associated with two variables $u_{k,l}$ and $v_{k,l}$, to record  its values before and after an activation function,
respectively.
The ReLU \cite{relu} is by far the most popular 
activation function for DNNs, according to which the \emph{activation 
value} of each node of hidden layers is defined as
\begin{equation}
    \label{eq:relu}
    v_{k,l}=ReLU(u_{k,l})=
    \begin{cases}
        u_{k,l} &\mbox{  if } u_{k,l}\geq 0 \\
            0 & \mbox{  otherwise}
    \end{cases}
\end{equation}
Each input node $n_{1,l}$ for $1\leq l\leq s_1$ is associated with a
variable $v_{1,l}$ and each output node $n_{K,l}$ for $1\leq l\leq s_K$ is
associated with a variable $u_{K,l}$, because no activation function is
applied on them.  We let $D_{L_k} = \mathbb{R}^{s_k}$ be the vector space
associated with layer $L_k$, one dimension for each variable $v_{k,l}$. 
Notably, every point $x\in D_{L_1}$ is an input.

Except for inputs, every node is connected to nodes in the
preceding layer by pre-trained parameters such that for all $k$ and $l$ with
$2 \leq k\leq K$ and  $1\leq l\leq s_k$, we have:
\begin{equation}
  \label{eq:sum}
  u_{k,l}=b_{k,l}+\sum_{1\leq h \leq s_{k-1}} w_{k-1, h, l}\cdot v_{k-1,h}
\end{equation}
where $w_{k-1,h,l}$ is the weight for the connection between
$n_{k-1,h}$ (i.e., the $h$-th node of layer $k-1$) and $n_{k,l}$
(i.e., the $l$-th node of layer $k$), and $b_{k,l}$ is the
so-called \emph{bias} for node $n_{k,l}$.  We note that this
definition can express both fully-connected functions and
convolutional functions.
The function $\phi_k$ is the combination of Equations (\ref{eq:relu}) and (\ref{eq:sum}). Owing to the use of the ReLU as in \eqref{eq:relu}, the behavior of a neural
network is highly non-linear. 

Finally, for any input, the DNN assigns a \emph{label}, that is,
the index of the node of output layer with the largest value:
$\mathit{label}=\mathrm{argmax}_{1\leq l\leq s_K}u_{K,l}$.
Let $\labels$ be the set of labels. 
\begin{example}
Figure \ref{fig:nn} is a simple DNN with four layers. 
Its input space is $D_{L_1}=\real^2$ where $\real$ is the set of real numbers.
\vspace{-.25cm}
\begin{figure}[htp!]
\centering
\def\layersep{1.8cm}
\scalebox{0.85}{
\begin{tikzpicture}[shorten >=1pt,->,draw=black!50, node distance=\layersep]
    \tikzstyle{every pin edge}=[<-,shorten <=1pt]
    \tikzstyle{neuron}=[circle,fill=black!25,minimum size=15pt,inner sep=0pt]
    \tikzstyle{input neuron}=[neuron, fill=green!50];
    \tikzstyle{output neuron}=[neuron, fill=red!50];
    \tikzstyle{hidden neuron}=[neuron, fill=blue!50];
    \tikzstyle{annot} = [text width=4em, text centered]

    \foreach \name / \y in {1,...,2}
        \node[input neuron, pin=left:$v_{1,\y}$] (I-\name) at (0,-\y) {};

    \foreach \name / \y in {1,...,3}
        \path[yshift=0.5cm]
            node[hidden neuron] (H1-\name) at (\layersep,-\y cm) {};

    \foreach \name / \y in {1,...,3}
        \path[yshift=0.5cm]
            node[hidden neuron] (H2-\name) at (\layersep*2,-\y cm) {};

    \node[output neuron,pin={[pin edge={->}]right:$u_{4,1}$}, right of=H2-2, yshift=0.5cm] (O1) {};
    \node[output neuron,pin={[pin edge={->}]right:$u_{4,2}$}, right of=H2-2, yshift=-0.5cm] (O2) {};

    \foreach \source in {1,...,2}
        \foreach \dest in {1,...,3}
            \path (I-\source) edge (H1-\dest);

    \foreach \source in {1,...,3}
        \foreach \dest in {1,...,3}
            \path (H1-\source) edge (H2-\dest);

    \foreach \source in {1,...,3}
         \path (H2-\source) edge (O1);

    \foreach \source in {1,...,3}
         \path (H2-\source) edge (O2);

    \node[annot,above of=H1-1, node distance=1cm] (hl1) {Hidden layer};
    \node[annot,above of=H2-1, node distance=1cm] (hl2) {Hidden layer};
    \node[annot,left of=hl1] {Input layer};
    \node[annot,right of=hl2] {Output layer};

    \node[annot, right of=H1-1, node distance=0.0cm] (hl1) {\small $n_{2,1}$};
    \node[annot, right of=H1-2, node distance=0.0cm] (hl1) {\small $n_{2,2}$};
    \node[annot, right of=H1-3, node distance=0.0cm] (hl1) {\small $n_{2,3}$};
    \node[annot, right of=H2-1, node distance=0.0cm] (hl1) {\small $n_{3,1}$};
    \node[annot, right of=H2-2, node distance=0.0cm] (hl1) {\small $n_{3,2}$};
    \node[annot, right of=H2-3, node distance=0.0cm] (hl1) {\small $n_{3,3}$};
\end{tikzpicture}
}
  \caption{A simple neural network}
  \label{fig:nn}
\end{figure}
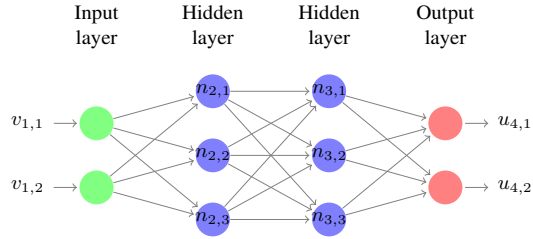

\vspace{-.65cm}
\end{example}
Given one particular input $x$, 
the DNN $\networks$ is
\emph{instantiated} and we use $\networks[x]$ to denote this instance of the
network.  In $\networks[x]$, for each node $n_{k,l}$, the values of the variables $u_{k,l}$ and $v_{k,l}$ are fixed and denoted by $u_{k,l}[x]$ and $v_{k,l}[x]$, respectively. 
%
Therefore, the activation
or deactivation of each ReLU operation in the network is also determined.  
We define
  \begin{equation}
    \label{eq:sign}
    \mathit{sign}_\networks (n_{k,l},x)=
    \begin{cases}
      +1 &\mbox{  if } u_{k,l}[x] = v_{k,l}[x] \\
      -1 & \mbox{  otherwise}
    \end{cases}
  \end{equation}
The subscript $\networks$ will be omitted when clear from the context. 
The classification label of $x$ is denoted by $\networks[x].\mathit{label}$.

\begin{example}\label{example:weights}
Let $\networks$ be a DNN whose architecture is given in Figure \ref{fig:nn}.  
Assume that the weights for the first three layers are given by
$
W_{1}=\scalemath{0.6}{
\begin{bmatrix}
  4 & 0 & -1\\
  1 & -2 & 1
\end{bmatrix}}$
and 
$W_{2}=\scalemath{0.6}{
\begin{bmatrix}
  2 & 3 & -1\\
  -7 & 6 & 4 \\
  1 & -5 & 9
\end{bmatrix}}
$
and that all biases are 0. When given an input 
$x=[0, 1]$, we get $sign(n_{2,1},x)=+1$, since 
$u_{2,1}[x]=v_{2,1}[x]=1$, and $sign(n_{2,2},x)=-1$, 
since $u_{2,2}[x] = -2 \neq 0 = v_{2,2}[x]$. 
\end{example}

We remark that, for simplicity of discussion, the definition focuses on
DNNs with fully connected layers. However, as shown in our experiments, 
our method can also be applied to other popular
DNN structures, such as convolutional and maxpooling layers, and sigmoid activation functions
used in the state-of-the-art DNNs.

%
%
%

\section{Adequacy Criteria for Testing DNNs}
\label{sec:criteria}

\subsection{Test Coverage and MC/DC}


%

%
A test adequacy criterion, or a test coverage metric, is used to quantify
the degree of adequacy to which the software is tested by a test suite with respect to a set
of test conditions. Throughout this paper, we use "criterion" and "metric" interchangeably.

Our criteria for DNNs are inspired by established practices in
software testing, in particular MC/DC test criterion\cite{HVCR2001}, but are
designed for the specific attributes of DNNs.
%
MC/DC is a method of measuring the extent to which safety-critical software has been adequately tested.
At its
core is the idea that if a choice can be made, all the possible factors
(conditions) that contribute to that choice (decision) must be tested.  For
traditional software, both conditions and the decision are usually Boolean
variables or Boolean expressions. 
\begin{example}\label{example:mcdc}
The decision
\begin{equation}\label{eq:mcdc}
d:\,\, ((a > 3) \lor (b = 0) ) \land (c \not= 4)
\end{equation}
contains the conditions: $(a>3)$, $(b=0)$ and $(c\not=4)$.  The
following four test cases provide 100\% MC/DC coverage:
\begin{enumerate}
\item $(a>3)$=\false, $(b=0)$=\true,  $(c \not= 4)$=\false
\item $(a>3)$=\true,  $(b=0)$=\false, $(c \not= 4)$=\true
\item $(a>3)$=\false, $(b=0)$=\false, $(c \not= 4)$=\true
\item $(a>3)$=\false, $(b=0)$=\true,  $(c \not= 4)$=\true
\end{enumerate}
The first two test cases already satisfy both  \emph{condition coverage}
(i.e., all possibilities of the conditions  are exploited) and 
\emph{decision coverage} (i.e., all possibilities of the decision  are
exploited).  The other two cases are needed because, for MC/DC, each
condition should evaluate to \true and \false at least once, and should
independently affect the decision outcome (e.g., the effect of the first condition can be seen by comparing cases 2 and 3).
\end{example}

\subsection{Decisions and Conditions in DNNs}

Our instantiation of the concepts ``decision'' and ``condition'' for DNNs is
inspired by the similarity between Equation~(\ref{eq:sum}) and
Equation~(\ref{eq:mcdc}) and the unique properties of DNNs.  
The information represented by nodes
in the next layer can be seen as a {summary}
(implemented by the layer function, the weights and the bias) of the
information in the current layer.
For example, it has been claimed that nodes in a deeper layer represent more complex
attributes of the input~\cite{YCNFL2015,olah2018the}.

We let $\setoffeatures_k \subseteq \mathcal{P}(L_k)$ be a set of subsets of nodes at layer $k$. 
Without loss of generality, each element of $\setoffeatures_k$, i.e., a subset of nodes in $L_k$, represents a \emph{feature} learned at layer $k$.
Therefore, 
\emph{the core idea of our criteria is to
ensure that not only the presence of a feature needs to be tested but also
the effects of less complex features on a more complex feature must be
tested}.
%
%
 We use $t_k = |\setoffeatures_k|$ to denote the number of {features} in $\setoffeatures_k$ and $\feature_{k,l}$ for $1\leq l\leq t_k$  the $l$-th {feature}. It is noted that the {features} can be  overlapping, i.e., $\feature_{k,l_1}\cap \feature_{k,l_2} \neq \emptyset$. 
%
We consider every {feature} $\feature_{k,l}$ for $2\leq k\leq K$
and $1\leq l\leq t_k$ a \emph{decision} and say that its \emph{conditions}
are those {features} connected to it in the layer $k-1$, i.e., 
$\{ \feature_{k-1,l'}~|~1\leq l' \leq t_{k-1}\}$. 

The use of {feature} generalises the basic building block in the DNN from a single node to a set of nodes.
A single node can be represented as a singleton set. In practice, the {feature} can be supported by the tensor implementation in popular machine learning libraries~\cite{tensorflow} and various feature extraction methods such as SIFT \cite{SIFT}, SURF \cite{SURF}, etc. 
To work with features, we extend the notations $u_{k,l}[x]$ and $v_{k,l}[x]$ for a node $n_{k,l}$ to a feature $\feature_{k,l}$ and write $\feature_{k,l}[x]$ and $\phi_{k,l}[x]$  for the  vectors before and after ReLU, respectively. 

\begin{definition}
\label{def:neuron-pair}
A feature pair $(\feature_{k,i}, \feature_{k+1,j})$ are two features in adjacent layers $k$
and $k+1$ such that $1\leq k< K$,  $1\leq i\leq t_k$ and $1\leq j\leq
t_{k+1}$.  Given a DNN $\networks$, we write $\neuronpairs(\networks)$ (or, simply $\neuronpairs$)
for the set of its feature pairs. We may also call $(\feature_{k,i}, \feature_{k+1,j})$ a neuron pair when both $\feature_{k,i}$ and $\feature_{k+1,j}$ are singleton sets.
\end{definition}

Our new criteria are defined by capturing different ways of instantiating the
changes of the conditions and the decision.  Unlike Boolean variables or
expressions, where it is trivial to define change, i.e., \true $\to$ \false
or \false $\to$ \true, in DNNs there are many different ways of
defining that a decision is \emph{affected} by the changes of the
conditions.  Before giving definitions for ``affected'' in
Section~\ref{sec:covering}, we start by clarifying when a feature ``changes''.



First, the change observed on a 
feature
can be
either a sign change or a value change.

\begin{definition}[Sign Change]\label{def:sc}
Given a feature $\feature_{k,l}$ and two test cases $x_1$ and $x_2$, 
the
sign change of $\feature_{k,l}$ is exploited by $x_1$ and $x_2$, denoted by
$sc(\feature_{k,l},x_1,x_2), iff$
\begin{itemize}
    \item $\mathit{sign}(n_{k,j},x_1) \neq
\mathit{sign}(n_{k,j},x_2)$ for all $n_{k,j}\in\feature_{k,l}$.
\end{itemize}  
Moreover, we~write $ nsc(\feature_{k,l},x_1,x_2)$ if
\begin{itemize}
    \item $\mathit{sign}(n_{k,j},x_1) =
\mathit{sign}(n_{k,j},x_2)$ for all $n_{k,j}\in\feature_{k,l}$.
\end{itemize}
\end{definition}

Note that $nsc(\feature_{k,l},x_1,x_2) \neq \neg sc(\feature_{k,l},x_1,x_2)$. 
Before preceeding to another kind of change called \emph{value change}, we need notation for value function. A value function is denoted by $g: \setoffeatures_k \times D_{L_1}\times D_{L_1}\rightarrow \{\true,\false\}$. Simply speaking, it  expresses the DNN developer's intuition (or knowledge) about what constitutes a significant change on the feature $\feature_{k,l}$, by specifying the difference between two vectors $\feature_{k,l}[x_1]$ and $\feature_{k,l}[x_2]$.  
We do not impose particular restrictions on the form of a value function, except that for practical reasons,  it needs to be evaluated efficiently. Here, we give a few examples. 

\begin{example}
For a singleton set $\feature_{k,l}=\{n_{k,j}\}$, the  function  $\valuefunction(\feature_{k,l}, x_1, x_2)$ can express 
$|u_{k,j}[x_1] - u_{k,j}[x_2]| \geq d$ (absolute change) or
$\frac{u_{k,j}[x_1]}{u_{k,j}[x_2]} > d \lor
\frac{u_{k,j}[x_1]}{u_{k,j}[x_2]} < 1/d$ (relative change).
It~can also express the constraint on one of the values $u_{k,j}[x_2]$, such
as $u_{k,j}[x_2] > d $ (upper boundary).
\end{example}
\begin{example}
For the general case, the  function $\valuefunction(\feature_{k,l}, x_1, x_2) $ can express the distance between two vectors  $\feature_{k,l}[x_1]$ and $\feature_{k,l}[x_2]$ by norm-based distances $\distance{\feature_{k,l}[x_1] -
\feature_{k,l}[x_2]}{p}  \leq d$ for a real number $d$ and a distance measure
$L^p$, or structural similarity distances such as SSIM \cite{WSB2003}. It can also express
constraints between nodes of the same layer, such as 
$\bigwedge_{j\neq i}v_{k,i}[x_1] \geq v_{k,j}[x_1]$. 
\end{example}

In general, the distance measure $L^p$ could be $L^1$ (Manhattan distance),
$L^2$ (Euclidean distance), $L^\infty$ (Chebyshev distance) and so on. 
We remark that there is no consensus on which norm is the best to use and,
furthermore, this is likely problem-specific.  Finally, we define value change as follows.

\begin{definition}[Value Change]\label{def:vc}
Given a feature $\feature_{k,l}$, two test cases $x_1$ and $x_2$, and a value function $\valuefunction$,  
the value change of $\feature_{k,l}$ is exploited by $x_1$ and $x_2$ w.r.t. 
$\valuefunction$, denoted by
$vc(\valuefunction,\feature_{k,l},x_1,x_2)$, if 
\begin{itemize}
    \item $\valuefunction(\feature_{k,l}, x_1,x_2)$=\true.
\end{itemize}
Moreover, we write $\neg
vc(\valuefunction,\feature_{k,l},x_1,x_2)$ when the condition is not satisfied.
\end{definition}


%
%
%
%
%
%

\subsection{Covering Methods}
\label{sec:covering}

In this section, we present a family of four methods to cover the causal changes
in a DNN that were just defined.

\begin{definition}[Sign-Sign Coverage, or SS Coverage]
  \label{def:ssc}
  A feature pair $\neuronpair=(\feature_{k,i},\feature_{k+1,j})$ is SS-covered by two test cases
  $x_1, x_2$, denoted by $\covered{SS}{}{(\neuronpair,x_1,x_2)}$, if the following
  conditions are satisfied by the  DNN instances $\networks[x_1]$ and $\networks[x_2]$: 
  \begin{itemize}
    \item $sc(\feature_{k,i},x_1,x_2)$ and  $nsc(P_k\setminus \feature_{k,i},x_1,x_2)$;
    \item $sc(\feature_{k+1,j},x_1,x_2)$.
  \end{itemize}
  where $P_k$ is the set of nodes in layer $k$.  
\end{definition}

SS coverage provides 
evidence that the sign change of a condition
feature $\feature_{k,i}$
independently affects the sign of the
decision feature $\feature_{k+1,j}$ of the next layer.  Intuitively, the first
condition says that the sign change of feature $\feature_{k,i}$ is exploited using
$x_1$ and $x_2$, without changing the signs of other non-overlapping features.  The second says that 
the sign change of feature $\feature_{k+1,j}$ is exploited using $x_1$
and $x_2$.

\begin{example}(Continuation of Example~\ref{example:weights})\label{example:SSc}
Given 
inputs $x_1=(0.1,0)$ and $x_2=(0,-1)$, we  compute the
activation values for each node as given in Table~\ref{tab:value-table}. 
Therefore, we have $sc(\{n_{2,1}\},x_1,x_2)$, $nsc(\{n_{2,2}\},x_1,x_2)$, $nsc(\{n_{2,3}\},x_1,x_2)$ and $sc(\{n_{3,1}\},x_1,x_2)$.  By
Definition~\ref{def:ssc}, the feature pair $(\{n_{2,1}\}, \{n_{3,1}\})$ is SS-covered
by $x_1$ and $x_2$.
\end{example}

\begin{table}
  \begin{center}
   \scriptsize
   \scalebox{1.0}{
    \begin{tabular}{|l|c|c|c|c|c|c|c|} \hline
        input
                  & $n_{2,1}$ & $ n_{2,2} $ & $ n_{2,3}$ & $ n_{3,1}$ & $n_{3,2}$ & $n_{3,3}$  \\\hline
      $(0.1,0)$   & $0.4$     & $0$ & $-0.1$ (0) & $0.8$ & $1.2$ & $-0.4$ (0) \\
      $(0,-1)$    & $-1(0)$   & $2$ & $-1$ (0) & $-14$ (0) & $12$ & $8$ \\\hline
      sign ch.    & sc        & $\neg$sc & $\neg$sc & sc & $\neg$sc & sc \\\hline\hline
      
      $(0,1)$     & $1$ & $-2$ (0) & $1$ & $3$ & $-2$ (0) & $8$ \\
      $(0.1,0.1)$ & $0.5$ & $-0.2$ (0) & 0 & $1$ & $1.5$ & $-0.5$ (0) \\\hline
      sign ch.    & $\neg$sc & $\neg$sc & $\neg$sc & $\neg$sc & sc & sc \\\hline\hline
      
      $(0,-1)$    & $-1$ (0) & $2$      & $-1$ (0)   & $-14$ (0)  & $12$  & $8$ \\
      $(0.1,-0.1)$& $0.3$    & $0.2$    & $-0.2$ (0) & $-0.8$ (0) & $2.1$ & $0.5$ \\\hline
      sign ch.    & sc       & $\neg$sc & $\neg$sc   & $\neg$sc   & $\neg$sc & sc \\\hline\hline
      
      $(0,1)$     & $1$      & $-2$ (0) & $1$   & $3$   & $-2$ (0) & $8$ \\
      $(0.1,0.5)$ & $0.9$    & $-1$ (0) & $0.4$ & $2.2$ & $0.7$ & $2.7$ \\\hline
      sign ch.    & $\neg$sc & $\neg$sc & $\neg$sc & $\neg$sc & sc & $\neg$sc \\\hline\hline
    \end{tabular}
    }
  \end{center}
  \caption{Activation values and sign changes for the input examples in Examples~\ref{example:SSc}, \ref{example:DS}, \ref{example:SV},  \ref{example:DV}. An entry can be of the form $v$, in which $v\geq 0$, or $u(v)$, in which $u < 0$, $v = 0$, or $sc$, denoting that the sign has been changed, or $\neg sc$, denoting that there is no sign change.}
  \label{tab:value-table}
\end{table}
%

SS coverage is 
close to MC/DC: instead of observing the change of a
Boolean variable (i.e., \true $\to$ \false or \false $\to$ \true), we
observe a sign change of a feature.
However, the behavior of a DNN has additional complexity
that is not necessarily captured by 
a~direct adoption of the MC/DC-style coverage to a DNN. 
\textbf{Subsequently, three additional coverage criteria are designed
to complement SS~coverage.}

First, the sign of $\feature_{k+1,j}$ can be altered between two test cases, 
even when none of the nodes $n_{k,i}$ in layer $k$ changes its sign. Note that $P_k$, the set of all nodes in layer $k$, is also a feature
and thus we write $nsc(P_k,x_1,x_2)$ to express that no sign change occurs for any of the nodes in layer $k$. 

\begin{definition}[Value-Sign Coverage, or VS Coverage]
  \label{def:dsc}
  Given a value function $\valuefunction$, a feature pair
  $\neuronpair=(\feature_{k,i},\feature_{k+1,j})$ is VS-covered by two test cases 
  $x_1, x_2$, denoted by $\covered{VS}{\valuefunction}{(\neuronpair,x_1,x_2)}$,
  if the following conditions are satisfied by the DNN instances $\networks[x_1]$
  and $\networks[x_2]$: 
  \begin{itemize}
    \item $vc(\valuefunction,\feature_{k,i},x_1,x_2)$ and $nsc(P_{k},x_1,x_2)$;
    \item $sc(\feature_{k+1,j},x_1,x_2)$.
  \end{itemize}
\end{definition}

Intuitively, the first condition describes the value change of nodes in
layer $k$ and the 
second 
requests the sign change of the feature
$\feature_{k+1,j}$.  Note that, in addition to $vc(\valuefunction,\feature_{k,i},x_1,x_2)$, 
we need $nsc(L_k,x_1,x_2)$, 
which asks for no sign changes
for any node at layer $k$. 
This is to ensure that the overall
change to the activations in layer $k$ is relatively small. 
%


\begin{example} (Continuation of Example~\ref{example:weights})\label{example:DS}
Given two inputs $x_1=(0,1)$ and $x_2=(0.1,0.1)$, by the computed activation
values in Table~\ref{tab:value-table},  we have $sc(\{n_{3,3}\},x_1,x_2)$ and
all nodes in layer 2 do not change their activation signs, i.e., $ nsc(\{n_{2,1},n_{2,2},n_{2,3}\},x_1,x_2)$.
%
Thus, by Definition~\ref{def:dsc}, $x_1$ and $x_2$ (subject to certain
value function ~$g$) can be used to VS-cover the feature pair
e.g., $(\{n_{2,1},n_{2,2}\},\{n_{3,3}\})$.
%
\end{example}






Until now, we have seen the sign change of a decision feature $\feature_{k+1,j}$ as
the equivalent of the change of a decision in MC/DC.  This view may still be
limited.  For DNNs, a key safety problem~\cite{szegedy2014intriguing} related to their
high non-linearity is that an insignificant (or imperceptible) change to the
input (e.g., an image) may lead to a significant change to the output (e.g.,
its label).  We expect that our criteria can guide test case
generation algorithms towards unsafe cases, by working with two
adjacent layers that are finer than the input-output relation.  We notice
that the label change in the output layer is the direct result of the
changes to the activation values in the penultimate layer.  Therefore, in
addition to the sign change, the change of the value of the decision feature
$\feature_{k+1,j}$ is also important.

\begin{definition}[Sign-Value Coverage, or SV Coverage]
\label{def:svc}
Given a value function $\valuefunction$, a feature pair
$\neuronpair=(\feature_{k,i},\feature_{k+1,j})$ is SV-covered by two test cases $x_1,
x_2$, denoted by $\covered{SV}{\valuefunction}{(\neuronpair,x_1,x_2)}$, if
the following conditions are satisfied by the DNN instances
$\networks[x_1]$ and $\networks[x_2]$:
  \begin{itemize}
    \item $sc(\feature_{k,i},x_1,x_2)$ and  $nsc(P_k \setminus \feature_{k,i},x_1,x_2)$;
    \item $vc(\valuefunction,\feature_{k+1,j},x_1,x_2)$ and $nsc(\feature_{k+1,j},x_1,x_2)$.
  \end{itemize}
\end{definition}

The first  condition is the same as that in
Definition~\ref{def:ssc}.  The difference is in the second condition, which now considers the feature value change $vc(\valuefunction,\feature_{k+1,j},x_1,x_2)$ with respect to a value function $\valuefunction$, by independently
modifying one its condition features' sign. Intuitively, SV Coverage
captures the significant change of a decision feature's value that complements the sign change case.

\begin{example} (Continuation of Example~\ref{example:weights})\label{example:SV}
Consider the feature pair $(\{n_{2,1}\}, \{n_{3,2}\})$. Given two inputs $x_1=(0,-1)$
and $x_2=(0.1,-0.1)$, by the computed activation values in
Table~\ref{tab:value-table}, we have $sc(\{n_{2,1}\},x_1,x_2)$ and $
nsc(\{n_{2,2},n_{2,3}\},x_1,x_2)$.
%
If, according to the function~$\valuefunction$,
$\frac{u_{3,2}[x_1]}{u_{3,2}[x_2]}\approx 5.71$ is a significant change, i.e., 
$g(u_{3,2}[x_1],u_{3,2}$ $[x_2])$=\true, 
then the pair $(\{n_{2,1}\}, \{n_{3,2}\})$ is SV-covered by $x_1$ and $x_2$.
\end{example}

Finally, we have the following definition by replacing the sign change of
the decision in Definition~\ref{def:dsc} with 
value change.

\begin{definition}[Value-Value Coverage, or VV Coverage]
  \label{def:dvc}
 Given two value functions $\valuefunction_1$ and $\valuefunction_2$, a feature pair $\neuronpair=(\feature_{k,i},\feature_{k+1,j})$ is VV-covered by two test cases $x_1, x_2$, denoted by $\covered{VV}{\valuefunction_1,\valuefunction_2}{(\neuronpair,x_1,x_2)}$, if the following conditions are satisfied by the  DNN instances $\networks[x_1]$ and $\networks[x_2]$: 
  \begin{itemize}
    \item $vc(\valuefunction_1, \feature_{k,i},x_1,x_2)$ and $nsc(P_k,x_1,x_2)$;
    \item $vc(\valuefunction_2,\feature_{k+1,j},x_1,x_2)$ and $nsc(\feature_{k+1,j},x_1,x_2)$.
  \end{itemize}
\end{definition}

Intuitively, VV coverage targets
scenarios in which there is no sign change for a condition feature, but the decision
feature's value is changed significantly.

\begin{example} (Continuation of Example~\ref{example:weights})\label{example:DV}
For any $i\in\{1,\ldots,3\}$, the feature pair $(\{\feature_{2,i}\},\{\feature_{3,3}\})$ 
are VV-covered by the inputs $x_1=(0,1)$ and
$x_2=(0.1,0.5)$, subject to the value functions $\valuefunction_1$ and
$\valuefunction_2$. As shown in Table~\ref{tab:value-table},
$\frac{u_{3,3}[x_1]}{u_{3,3}[x_2]}\approx 2.96$,
for all $i\in \{1,\ldots,3\}: nsc(\{n_{2,i}\},x_1,x_2)$ and $nsc(\{n_{3,3}\},x_1,x_2)$.
\end{example}

\subsection{Test Conditions, Test Suites and Test Criteria}

By utilising  the covering methods defined in Section~\ref{sec:covering},
we now are able to instantiate the test conditions, test suites and test criteria for DNNs.
Let
$F=\{\covered{SS}{}{},\covered{VS}{\valuefunction}{},\covered{SV}{\valuefunction}{},\covered{VV}{\valuefunction_1,\valuefunction_2}{}\}$
be a set of covering methods. Given a DNN $\networks$ and a  covering method $f\in F$, a test condition set  is characterised by the pair 
$(f,\neuronpairs)$ that asks for the coverage of corresponding causal changes on feature pairs in $\neuronpairs$ according to $f$.

 
 


Given a DNN $\networks$, a test suite $\testsuites$ is a finite set of inputs, i.e., $\testsuites\subseteq D_{L_1}$. 
Ideally, 
we run a test case generation algorithm to find a test suite $\testsuites$ such that %
  \begin{equation}
  \forall \neuronpair\in \neuronpairs ~\exists~ x_1,x_2 \in \testsuites: f(\neuronpair,x_1,x_2)
 \end{equation} 
In practice, we might want to
compute the degree to which the test conditions are satisfied by a generated test suite $\testsuites$.


\begin{definition}[Test Criterion]
Given a DNN $\networks$, a test condition set by $(f,\neuronpairs)$ and a test suite $\testsuites$, 
the test criterion $\metric_{f}(\networks,\testsuites)$ is defined as follows: 
\begin{equation}
  \label{eq:madc}
  \metric_f(\networks,\testsuites)=\frac{|\{\neuronpair \in \neuronpairs | \exists x_1,x_2 \in \testsuites: f(\neuronpair,x_1,x_2)\}|}{|\neuronpairs|}
\end{equation}
\end{definition}
 
That is, it computes the percentage of the feature pairs that are
covered by test cases in $\testsuites$ 
with respect to 
the covering
method~$f$.


\commentout{

\begin{table}
  \begin{center}
    \begin{tabular}{|l|l|l|} 
    \hline
                     &   sign change   &   value change  \\
                     &    (Def.~\ref{def:sc})  &    (Def.~\ref{def:vc}) \\ 
                     \hline
      sign change (Def.~\ref{def:sc})          &   $M_{\covered{SS}{}{}}(O,\testsuites)$  &   $M_{\covered{VS}{\valuefunction}{}}(O,\testsuites)$           \\\hline
      value change (Def.~\ref{def:vc}) &   $M_{\covered{SV}{\valuefunction}{}}(O,\testsuites)$  &  $M_{\covered{VV}{\valuefunction_1,\valuefunction_2}{}}(O,\testsuites)$      \\\hline
    \end{tabular}
    \caption{A set of adequacy criteria for testing the DNNs, inspired by the MC/DC structural
    coverage criteria. The columns are changes to the conditions (i.e., the activations of features
    in layer $k$), and the rows are changes to the decision (i.e., the activation of a feature in 
    layer $k+1$)}
    \label{tab:criteria}
  \end{center}
\end{table}

}

Finally, instantiating $f$ with covering methods in $F$, we obtain four test
criteria $M_{\covered{SS}{}{}}(\networks,\testsuites)$,
$M_{\covered{VS}{\valuefunction}{}}(\networks,\testsuites)$,
$M_{\covered{SV}{\valuefunction}{}}(\networks,\testsuites)$ and
$M_{\covered{VV}{\valuefunction_1,\valuefunction_2}{}}$ $(\networks,\testsuites)$. 


\commentout{
\section{Generalised Condition-Decision Neuron Pairs}

\textbf{No. We do not need this section. I guess the easiest way to modify the paper 
is to redefine sign change in Definition \ref{def:sc} from a single
neuron to a set of neurons}

The core idea of our proposed criteria is to capture condition changes that cause the
change of the decision neuron. In particular, we define sign change for each condition
neuron. However, we must realise that for large DNN models with millions of neurons, 
a single condition sign change by itself may be not able to affect the decision neuron to the
extent of sign change or value change. Thus, 
}

\section{Comparison with Existing Structural Test Criteria}\label{sec:comparison}

So far, there have been a few proposals for structural test coverage criteria for DNNs.
In this part, we compare our criteria with them, 
including the safety coverage ($M_{\covered{S}{}{}}$) \cite{WHK2018}, 
neuron coverage ($M_{\covered{N}{}{}}$) \cite{PCYJ2017} and several of its extensions 
in \cite{ma2018deepgauge} such as neuron boundary coverage ($M_{\covered{NB}{}{}}$), 
multisection neuron coverage ($M_{\covered{MN}{}{}}$) and top  neuron coverage ($M_{\covered{TN}{}{}}$).
While \cite{PCYJ2017} and \cite{WHK2018} have been authored slightly ahead of ours, our criteria have been developed in parallel with \cite{ma2018deepgauge}. 


A metric $M_1$ is said to be weaker than another metric $M_2$, denoted by $M_1\preceq M_2$, iff for any given test suite $\testsuites$ on 
$\networks$, we have $M_1(\networks,\testsuites)<1 $ implies  $ M_2(\networks,\testsuites)<1$. 
%
For instance, as shown in Example~\ref{example:mcdc}, decision coverage and condition coverage are weaker than MC/DC, since MC/DC cannot be covered before all decisions and conditions are covered.

The introduction of the feature relation in this work is very powerful:
1) the criteria in this paper are stronger than those in \cite{PCYJ2017} and \cite{ma2018deepgauge}, which only consider individual neurons' activation statuses,
and 2) it is non-trivial for the safety coverage in \cite{WHK2018}, which is comparable to the traditional path coverage that asks to cover every
program execution path, to cover all test conditions of our criteria.

%

In the following, we uniformly formalise the criteria in \cite{PCYJ2017,WHK2018,ma2018deepgauge} based on notations in this paper and
we will define $M_f(\networks,\testsuites)$ for $f\in \{\covered{N}{}{},\covered{S}{}{},\covered{NB}{}{},\covered{MN}{}{},\covered{TN}{}{}\}$.

\begin{definition} [Neuron Coverage] 
\label{def:neuron-coverage}
A node $n_{k,i}$ is neuron covered by a test case $x$, denoted by $\covered{N}{}(n_{k,i},x)$, if $\mathit{sign}(n_{k,i},x)= +1$. 
\end{definition}

Given the definition, the neuron coverage asks that each neuron $n_{k,i}$ must be activated at least once by some test input $x$: $\mathit{sign}(n_{k,i},x)= +1$.

The neuron coverage was later generalised in \cite{ma2018deepgauge} to cover more fine-grained neuron activation statuses, including the boundary value for
a neuron's activation.
For simplicity, we only consider upper bounds when working with neuron boundary coverage. Given a node $n_{k,i}$ and a training dataset $X$, we let $v_{k,i}^u=\max_{x\in X}v_{k,i}[x]$ be its  maximum value  over the inputs in $X$. 

\begin{definition}  [Neuron Boundary Coverage] 
A node $n_{k,i}$ is neuron boundary covered by a test case $x$, denoted by $\covered{NB}{}(n_{k,i},x)$, if $v_{k,i}[x] > v_{k,i}^u$. 
\end{definition}

Let $rank(n_{k,i},x)$ be the rank of $v_{k,i}[x]$ among those values of the nodes at the same layer, i.e.,  $\{v_{k,j}[x]~|~1\leq j\leq s_k\}$. 
\begin{definition}  [Top Neuron Coverage] 
For $1\leq m\leq s_k$, a node $n_{k,i}$ is top-$m$ neuron covered by  $x$, denoted by $\covered{TN}{m}(n_{k,i},x)$, if $rank(n_{k,i},x) \leq m$. 
\end{definition}

Let $v_{k,i}^l=\min_{x\in X}v_{k,i}[x]$. We can split the interval $I_{k,i} = [v_{k,i}^l,v_{k,i}^u]$ into $m$ equal sections and let $I_{k,i}^j$ be the $j$th section. 

\begin{definition}  [Multisection Neuron Coverage] 
Given $m\geq 1$, a node $n_{k,i}$ is $m$-multisection neuron covered by  a test suite $\testsuites$, denoted by $\covered{MN}{m}(n_{k,i},\testsuites)$, if $\forall 1\leq j\leq m\exists x\in \testsuites: v_{k,i}[x] \in I_{k,i}^j$, i.e., all sections are covered by some test cases. 
\end{definition}

Given $f\in \{\covered{N}{},\covered{NB}{},\covered{TN}{m}\}$ and   the set ${\cal H}(\networks)$ of hidden nodes in  $\networks$, their associated test criterion can
be then defined as follows
\begin{equation}
M_f(\networks, \testsuites) = \frac{|\{n \in {\cal H}(\networks)~|~\exists x\in\testsuites: f(n,x)\}|}{|{\cal H}(\networks)|}
\end{equation}
$M_{\covered{MN}{m}{}}(\networks,\testsuites)$ can be obtained by a simple adaptation. 


We can fnd out that the criteria in \cite{PCYJ2017,WHK2018} are special cases of our criteria
(with a suitable value function $\valuefunction$).
As an example, the ``weaker than'' relationship between neuron coverage and SS coverage is proved in the
lemma below.

\begin{lemma}\label{lemma:nss}
$M_{\covered{N}{}{}} \preceq M_{\covered{SS}{}{}}$.
\end{lemma}
\begin{proof}
Note that, for every hidden  node $n_{k,j}\in {\cal H}(\networks)$, there exists a feature pair $ (\{n_{k-1,i}\},\{n_{k,j}\})\in \neuronpairs(\networks)$ for any $1\leq i\leq s_{k-1}$. Then, by Definition~\ref{def:ssc}, we have $sc(\{n_{k,j}\},x_1,x_2)$, which by Definition~\ref{def:sc} means that $\mathit{sign}(n_{k,j},x_1) \neq
\mathit{sign}(n_{k,j},x_2)$. That is, either $\mathit{sign}(n_{k,j},x_1)=+1$ or $\mathit{sign}(n_{k,j},x_2)=+1$. Therefore, if $n_{k,j}$ is not covered in a test suite $\testsuites_1$ for neuron coverage, none of the pairs $ (\{n_{k-1,i}\},\{n_{k,j}\})$ for $1\leq i\leq s_{k-1}$ is covered in a test suite $\testsuites_2$ for SS coverage.
\end{proof}

Figure \ref{fig:relationship} gives a diagrammatic summary of the  relations between all existing structural test coverage 
criteria for DNNs. 
The arrows represent the ``weaker than'' relations. 
The complete proofs are in the appendix. 
As shown in Figure \ref{fig:relationship}, our criteria require more test cases  to be generated
than those in \cite{PCYJ2017,ma2018deepgauge}, and therefore can lead to more intensive testing. 
%

\begin{figure}[!htb]
    \centering
    \includegraphics[scale=0.35]{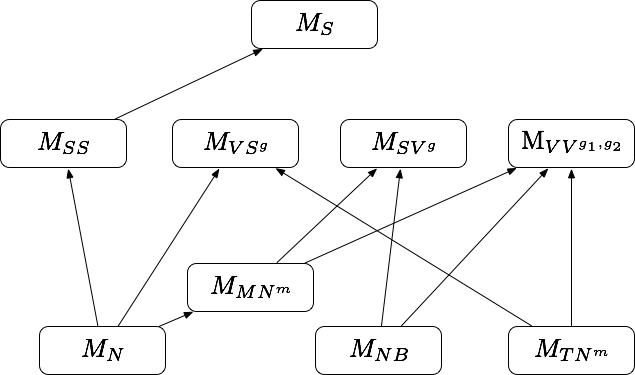}
    \caption{Relationship between DNN structural test criteria}
    \label{fig:relationship}
\end{figure}

On the other hand, as indicated in Figure \ref{fig:relationship},
 SS coverage is weaker than safety coverage~\cite{WHK2018}. 
In \cite{WHK2018},
the input space is discretised with
a set of hyper-rectangles, and then one test case is generated for each
hyper-rectangle. Such a scheme is computationally intractable due to the high-dimensionality of DNNs.
The testing approach in this paper is more practical. 
%
 
\begin{definition} [Safety Coverage]
Let each hyper-rectangle $rec$ contain those inputs with the same pattern
of ReLU, i.e., for all $x_1,x_2\in rec$ we have
$\mathit{sign}(n_{k,l},x_1)=\mathit{sign}(n_{k,l},x_2)$ for all $n_{k,l}\in
{\cal H}(\networks)$.  A~hyper-rectangle $rec$ is covered by a test
case $x$, denoted by $\covered{S}{}(rec,x)$, if $x\in rec$.
\end{definition}

Let $Rec(\networks)$ be the set of hyper-rectangles. Then 
\begin{equation}M_{\covered{S}{}{}}(\networks, \testsuites) = \frac{|\{rec \in Rec(\networks)~|~\exists x\in \testsuites: \covered{S}{}(rec,x)\}|}{|Rec(\networks)|}
\end{equation}
%



\section{Automated Test Case Generation}
\label{sec:test-gen}

\newcommand{\constraints}{\mathcal{C}}

We conjecture that the criteria proposed above achieve a good balance
between their ability to guide test case generation towards relevant cases
and computational cost. To show this hypothesis, we now apply our criteria with
two different test case generation approaches for DNNs.

The test conditions required by our criteria exhibit particular combinations
between the condition feature and the decision feature, and it is not
trivial to generate test cases for them.  Due to the lack of awareness of
the feature relation, testing methods in \cite{PCYJ2017,WHK2018,ma2018deepgauge}
cannot be directly used to generate tests for our criteria. Also, as 
pointed out in~\cite{odena2018tensorfuzz}, 
random test case generation is prohibitively inefficient for DNNs.  
Meanwhile, the symbolic encoding in the concolic testing method in \cite{sun2018concolic}
is expressive enough to encode test conditions defined by our criteria and 
is suitable for small to medium-sized DNNs.
Furthermore, in this section, we also present a new test case generation algorithm based on gradient descent (GD) 
search, which scales to large DNNs.
%
%

\subsection{Test Oracle}\label{sec:safety}

An oracle in software testing is a mechanism to detemine whether a test has passed or failed.
The DNN $\networks$ represents a function $\hat{\mathcal{F}}(x)$, which
approximates $\mathcal{F}(x): D_{L_1} \to \labels$ that models perfect human
perception capability.  Therefore, the ultimate safety requirement is that
for all test cases $x\in D_{L_1}$, we have $\hat{\mathcal{F}}(x) =
\mathcal{F}(x)$.  However, such a requirement is not practical because of the
large number of inputs in $D_{L_1}$ and the high cost of asking humans to
label images.
A pragmatic compromise, as done in many other works
including~\cite{szegedy2014intriguing,HKWW2017}, is to use the following oracle as a proxy.

\begin{definition}[Oracle]
\label{def:requirement}
Given a finite set $X$ of correctly labeled inputs, an input $x'$ passes the
oracle if there exists some $x\in X$ such that $x$ and $x'$ are close enough
and $\hat{\mathcal{F}}(x') =
\hat{\mathcal{F}}(x)$.
\end{definition}

Ideally, the question of whether two inputs $x$ and $x'$ are close enough is to be
answered according to the human perception.  In practice, this is approximated by
various approaches, including norm-based distance measures.
%
%
Specifically, given the norm $L^p$ and an upper bound $b$ for the distance,
we say that two inputs $x$ and $x'$ are close iff $\distance{x-x'}{p} \leq
b$. We write $\mathit{close}(x,x')$ for this relation.
A pair of inputs that satisfies this definition are called
\emph{adversarial examples} if the label assigned to them by the
DNN differs.

The choice of $b$ is problem-specific.  In our experiments, we evaluate the
distribution of adversarial examples with respect to the distance (as
illustrated in Figure~\ref{fig:ss-distance-map} for one of the criteria). 
The use of this oracle focuses on adversarial examples in the DNN.
There may exist other ways to define a test oracle for DNNs, and our criteria
are independent from its particular definition.

%


\subsection{Test Case Generation with LP}
\label{sec:lp}

%

We first adopt the concolic testing approach in \cite{sun2018concolic}
to generate test cases that satisfy the test conditions defined by our criteria.
In \cite{sun2018concolic}, test conditions are symbolically encoded
using an linear programming (LP) model that is solved to obtain new test cases.
Specifically, the LP-based approach fixes a particular pattern of node
activations according to a given input $x$. 

%
%


Though the overall behaviour of a DNN is highly non-linear, due to the use of e.g., the ReLU activation function,
when the DNN is instantiated with a particular input, the activation pattern is fixed, and this corresponds
to an LP model.
\paragraph{LP model of a DNN instance}\label{sec:encoding}

The variables used in the LP model are distinguished in \textbf{bold}.  All
variables are real-valued.  Given an input $x$, the input variable $\mathbf{x}$,
whose value is to be synthesized with LP, is required to have the identical
activation pattern as $x$, i.e., $\forall n_{k,i}: sign(n_{k,i},\mathbf{x})=sign(n_{k,i},x)$.

We use variables $\mathbf{u_{k,i}}$ and $\mathbf{v_{k,i}}$ to denote the values
of a node $n_{k,i}$ before and after the application of ReLU, respectively.
Then, we have the set $\constraints_1[x]$ of constraints to encode
ReLU operations for a network instance, where $\constraints_1[x]$ is given as:
\begin{equation}
\label{eq:lp-dir1}
\scalemath{0.9}{
\begin{array}{l}
 \{\mathbf{u_{k,i}}\geq 0 \wedge \mathbf{v_{k,i}}=\mathbf{u_{k,i}} ~|~ sign(n_{k,i},x)\geq 0, k\in [2,K), i\in [1\dots s_k] \} \\
\cup 
  \{\mathbf{u_{k,i}}< 0 \wedge \mathbf{v_{k,i}}=0~|~sign(n_{k,i},x)<0, k\in [2,K), i\in [1\dots s_k]\}
\end{array}
}
\end{equation}

Note, the activation values $\mathbf{u}_{k,i}$ of each node is
determined by the activation values $\mathbf{v}_{k-1,j}$ of those nodes in
the prior layer. This is defined as in Equation \eqref{eq:sum}. Therefore,
we add the following set of constraints,  $\constraints_2[x]$,
as a symbolic encoding of nodes' activation values.
\begin{equation} \label{eq:lp-v}
\scalemath{0.85}{
\{  \mathbf{u}_{k,i}=\sum_{1\leq j \leq s_{k-1}} \{{w}_{k-1, j, i}\cdot \mathbf{v}_{k-1,j}\} + b_{k,i} ~|~ k\in [2,K), i\in [1\dots s_k]\}
}
\end{equation}
The resulting LP model $\constraints[x] =
\constraints_1[x]\cup \constraints_2[x]$  represents \emph{a symbolic set of
inputs} that have the identical activation pattern as $x$.
Further, we can specify some optimisation objective $obj$ and call an LP solver
to find the optimal $\textbf{x}$ (if one exists). In concolic testing, each time the DNN
is instantiated with a concrete input $x_1$, the corresponding partial activation pattern serves
as the base for the LP modeling, upon which a new test input $x_2$ may be found that satisfies the specified test condition.

\subsection{Test Case Generation: a Heuristic Search }

The LP optimisation in Section \ref{sec:lp} provides a strong
guarantee that is able to return an input pair as long as one exists. However,
its scalability depends on the efficiency of LP solvers, and it is not trivial to apply such
a testing method to large-scale DNNs with millions of neurons. In this part, we instead develop
a heuristic algorithm  based on gradient search. Note that, it has been widley shown that following gradient change is efficient in finding bugs in DNNs and has been utilised in existing 
DNN testing methods (eg., \cite{PCYJ2017,ma2018deepgauge,yaghoubi-hscc}). 
\vspace{-0.25cm}
\begin{algorithm}[!htp]
  \caption{$get\_input\_pair(f, \feature_{k,i},\feature_{k+1,j})$}

  \label{algo:aac-get-input-pair-search}
  \begin{algorithmic}
    \For{each $x_1 \in data\_set$}
      \State sample an input  $x_2$ and a positive number $\epsilon$ 
      \For{a bounded number of steps}
        \If{$f((\feature_{k,i},\feature_{k+1,j}),x_1, x_2)$} \Return{$x_1,x_2$}  \EndIf
        \State update $\epsilon$       
        \If{$\neg f^{widen}((\feature_{k,i},\feature_{k+1,j}),x_1, x_2)$} $x_2\leftarrow x_2-\epsilon\cdot\nabla\hat{\mathcal{F}}(x_2)$
        \Else {$\,\,\,x_2\leftarrow x_2+\epsilon\cdot\nabla\hat{\mathcal{F}}(x_2)$}
        \EndIf 
      \EndFor
    \EndFor
    \State \Return{None, None}
  \end{algorithmic}
\end{algorithm}
\vspace{-0.25cm}
The algorithm, depicted in Algorithm \ref{algo:aac-get-input-pair-search}, is 
used to find an input pair $x_1, x_2$ such that the test condition 
of the covering method, $f$, over the feature pair, $\neuronpair=(\feature_{k,i},\feature_{k+1,j})$,
is satisfied; that is, $f(\neuronpair, x_1, x_2)$ is $\true$.
We use $f^{widen}(\neuronpair, x_1, x_2)$ for a widened version
of the testing condition $f$, such that all its predicates on the features $\feature_{k,i}$
and $\feature_{k+1,j}$ are eliminated. It is supposed that $x_1$ is given, and intuitively starting from  an input, $x_2$, if feature changes other than $\feature_{k,i}$ and 
$\feature_{k+1,j}$ do not meet the requirements of $f$, $x_2$ is moved closer to $x_1$, by following
the gradient descent: $x_2\leftarrow x_2-\epsilon\cdot\nabla\hat{\mathcal{F}}(x_2)$, as an attempt to counteract such changes. This applies to the case when the activation sign changes on other condition features. Otherwise, the change between $x_1$ and $x_2$ can only exploit a subset of 
predicates (in the testing condition) from the given feature pair, and we update $x_2$
following the gradient ascent. The algorithm's gradient change follows an adaptive manner that
comprises of a local search to update $x_2$ at each step, and a simple strategy for 
the overall search direction to move closer or further, with respect to $x_1$. 
In our implementation, we apply the FGSM (Fast Gradient Sign Method) \cite{FGSM}
 to initialise $x_2$ and $\epsilon$, and use a binary search scheme to update 
$\epsilon$ at each step.

As a heuristic, the algorithm works when there exists two inputs $x_1$ and $x_2$ s.t. $x_1$ is from the given ``data\_set'', $x_2$ is an input along the gradient search direction, and $(x_1,x_2)$ satisfies the specified
test condition. 

\definecolor{Gray}{gray}{0.9}
\definecolor{LightCyan}{rgb}{0.88,1,1}
\newcolumntype{g}{>{\columncolor{Gray}}r}

\section{Experiments}
\label{sec:experiments}
We conduct experiments using the well-known MNIST {Handwritten Image Dataset}~\cite{lecun1998gradient},
the CIFAR-10 dataset~\cite{krizhevsky2009learning} on small images
and the ImageNet benchmark~\cite{imagenet} from the large-scale visual recognition challenge.  
For clarity, our experiments are classified into four classes: \emph{\ding{192}~bug finding
\ding{193}~DNN safety statistics \ding{194}~testing efficiency \ding{195}~DNN internal structure
analysis}, and results will be labeled correspondingly. We also explain the relation
between our criteria and the existing ones.

In our implementation, 
the objective $\min{\distance{x_2\!-\!x_1}{\infty}}$ is used in all LP
calls, to find good adversarial examples with respect to the test
coverage conditions.  Moreover, we use 
$\valuefunction = \frac{u_{k+1,j}[x_2]}{u_{k+1,j}[x_1]}\geq \sigma$
%
%
with $\sigma=2$ for $g$ in $SV^g$ and $\sigma=5$ for
$VV^{\valuefunction_1, \valuefunction_2}$ (with respect to $\valuefunction_2$).  
We admit that such choices are experimental. For generality and to
speed up the experiments, we leave the value function $\valuefunction_1$ unspecified. 
Providing a specific $\valuefunction_1$ may require more effort to find an
$x_2$ (because $\valuefunction_1$ is an additional constraint), but the
resulting $x_2$ can be better.
\begin{figure}[!htb]
    \centering
    \def\arraystretch{1.2}
    \scalebox{0.9}{
    \begin{tabular}{|c|c|rg|rg|rg|rg|rg|} \hline
      & hidden layers & $M_{\covered{SS}{}{}}$ & $AE_{\covered{SS}{}{}}$ &  $M_{\covered{VS}{\valuefunction}{}}$ & $AE_{\covered{VS}{\valuefunction}{}}$ &  $M_{\covered{SV}{g}{}}$ & $AE_{\covered{SV}{g}{}}$ &  $M_{\covered{VV}{\valuefunction_1,\valuefunction_2}{}}$ & $AE_{\covered{VV}{\valuefunction_1, \valuefunction_2}{}}$ \\\hline
      $\networks_1$ & 67x22x63       & 99.7\% & 18.9\% & 100\% & 15.8\% & 100\% &6.7\% & 100\% & 21.1\%  \\\hline
      $\networks_2$ & 59x94x56x45    & 98.5\% &  9.5\% & 100\% & 6.8\% & 99.9\% & 3.7\% & 100\% & 11.2\% \\\hline
      $\networks_3$ & 72x61x70x77    & 99.4\% &  7.1\% & 100\% & 5.0\% & 99.9\% & 3.7\% & 98.6\% & 11.0\%     \\\hline
      $\networks_4$ & 65x99x87x23x31 & 98.4\% &  7.1\% & 100\% & 7.2\% & 99.8\% & 3.7\% & 98.4\% & 11.2\%  \\\hline
      $\networks_5$ & 49x61x90x21x48 & 89.1\% & 11.4\% & 99.1\%& 9.6\%& 99.4\%& 4.9\%& 98.7\%& 9.1\% \\\hline
      $\networks_6$ & 97x83x32       &100.0\% &  9.4\% & 100\%& 5.6\%& 100\%& 3.7\%& 100\%& 8.0\%  \\\hline
      $\networks_7$ & 33x95x67x43x76 & 86.9\% &  8.8\% & 100\%& 7.2\%& 99.2\%& 3.8\%& 96\%& 12.0\% \\\hline
      $\networks_8$ & 78x62x73x47    & 99.8\% &  8.4\% & 100\%& 9.4\%& 100\%& 4.0\%& 100\%& 7.3\% \\\hline
      $\networks_9$ & 87x33x62       &100.0\% & 12.0\% & 100\%& 10.5\%& 100\%& 5.0\%& 100\%& 6.7\% \\\hline
   $\networks_{10}$ & 76x55x74x98x75 & 86.7\% &  5.8\% & 100\%& 6.1\%& 98.3\%& 2.4\%& 93.9\%& 4.5\% \\\hline

    \end{tabular}
    }
    \captionof{table}{Coverage results on ten DNNs} 
    \label{tab:results}
\end{figure}
\vspace{-.75cm}
\begin{figure}[!htb]
\centering
\begin{tabular}{cccc}
  \subfloat[9 $\rightarrow$ 8]{
    \includegraphics[width=0.175\columnwidth]{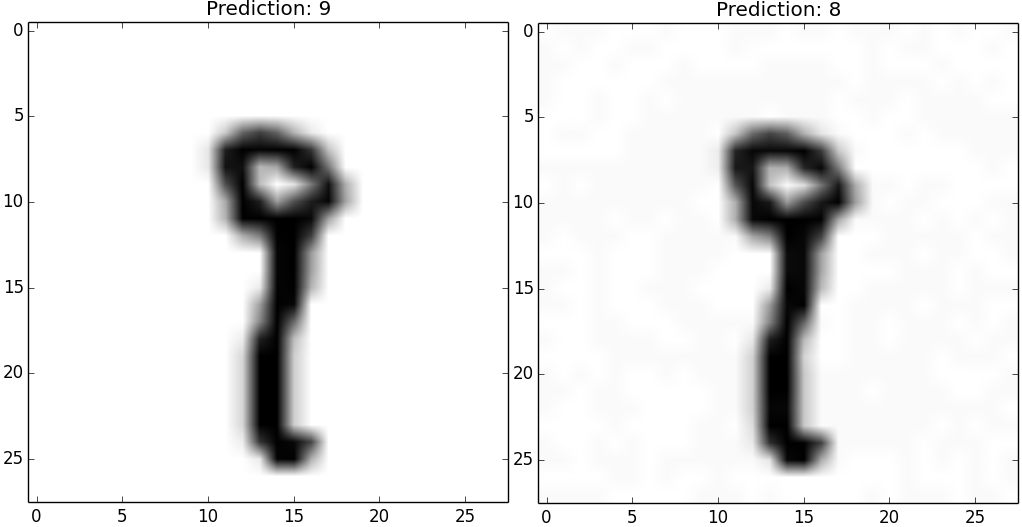}
    \label{fig:ad1}
  }
  &
  \subfloat[8 $\rightarrow$ 2]{
    \includegraphics[width=0.175\columnwidth]{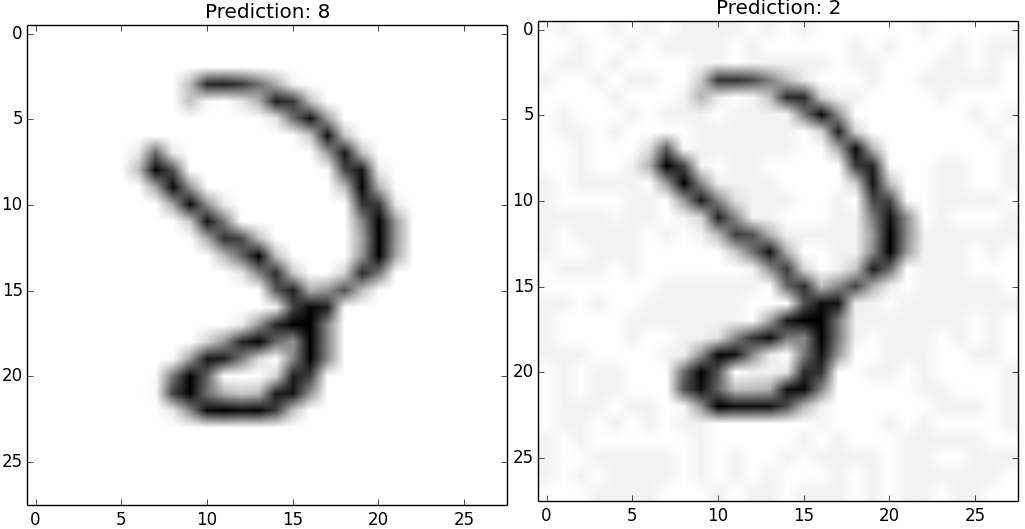}
    \label{fig:ad2}
  }
  & 
  \subfloat[1 $\rightarrow$ 7]{
    \includegraphics[width=0.175\columnwidth]{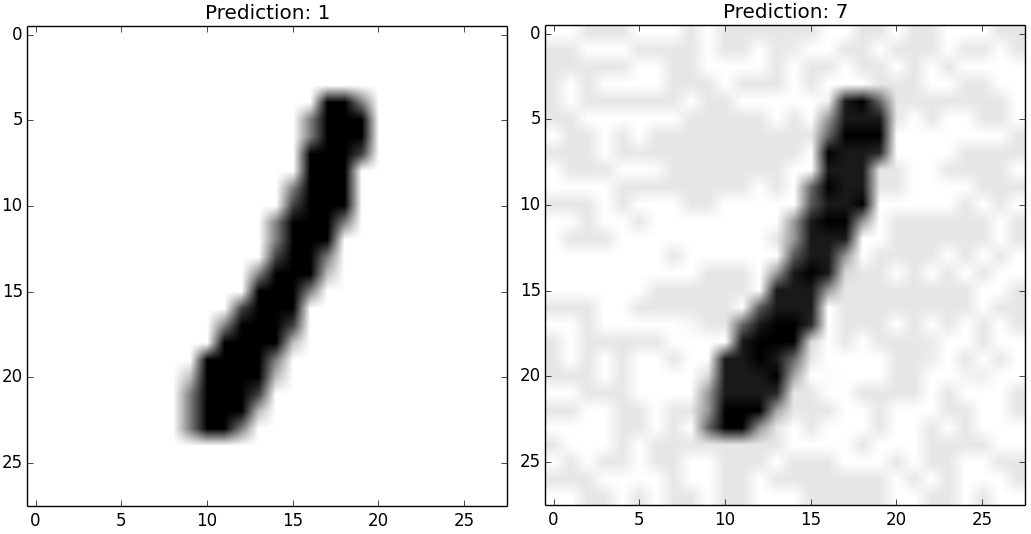}
    \label{fig:ad3}
  }
  &
  \subfloat[0 $\rightarrow$ 9]{
    \includegraphics[width=0.175\columnwidth]{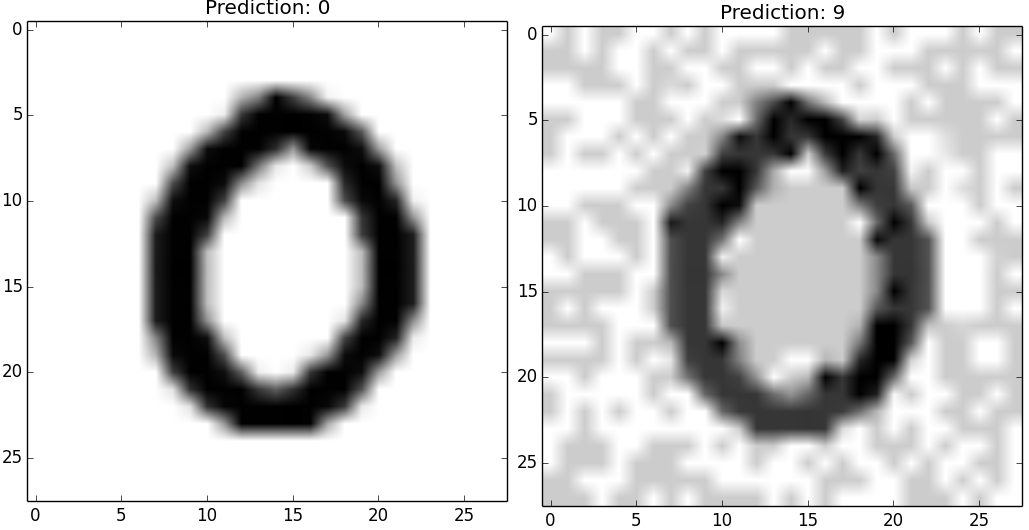}
    \label{fig:ad4}
  }
\end{tabular}
\caption{Selected adversarial examples for MNIST} 
\label{fig:ads}
\end{figure}

\subsection{MNIST}
We randomly generate, and then train, a set of ten fully connected DNNs, such that each network has an accuracy of at
least $97.0\%$ on the MNIST validation data. 
The detailed network structure, and the number of neurons per layer, are given in Table~\ref{tab:results}. Every DNN input has
been normalised into $[0,1]^{28\times 28}$.
Experiments were conducted on a MacBook Pro
(2.5\,GHz Intel Core i5, 8\,GB  memory).

We apply the  covering method defined in Section~\ref{sec:criteria}. 
Besides the coverage $M_f$, we also measure the
percentage of adversarial examples among all test pairs in the test suite,
denoted by $AE_f$. \emph{Thanks to the use of LP optimisation, the feature in this part
is fine-grained to single neuron level.} That is, each feature pair is in fact a neuron pair.

\begin{figure}[!htb]
\centering
\includegraphics[width=0.55\linewidth]{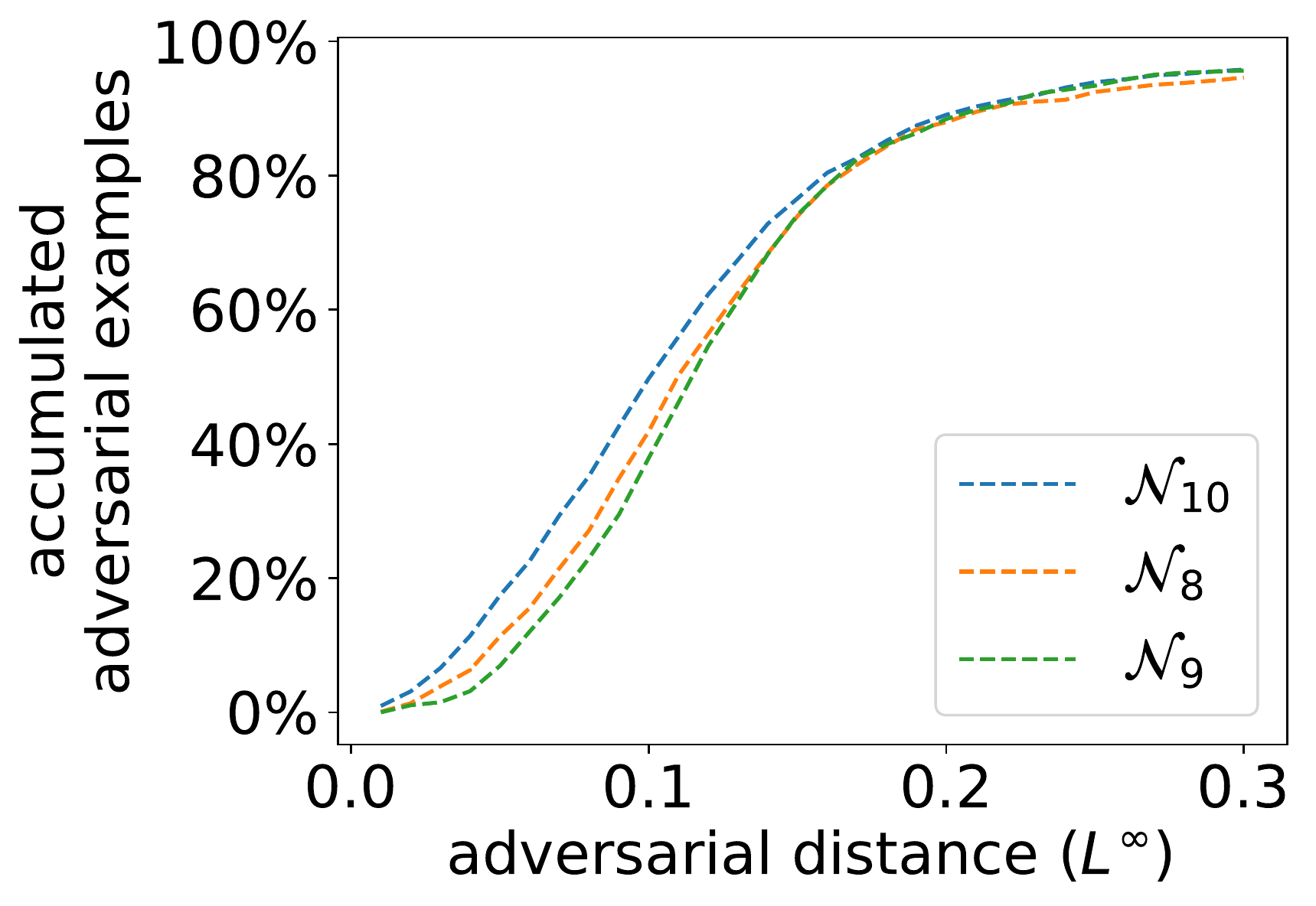}
\caption{Adversarial example curves that record the accumulated percentage of
adversarial examples that fall into each distance:
the adversarial distance measures the distance between an adversarial example and the original input}
\label{fig:ss-distance-map}
\end{figure}
\begin{figure}[!htb]
\centering
\begin{tabular}{cc}
  \subfloat[]{
    \includegraphics[width=0.45\linewidth]{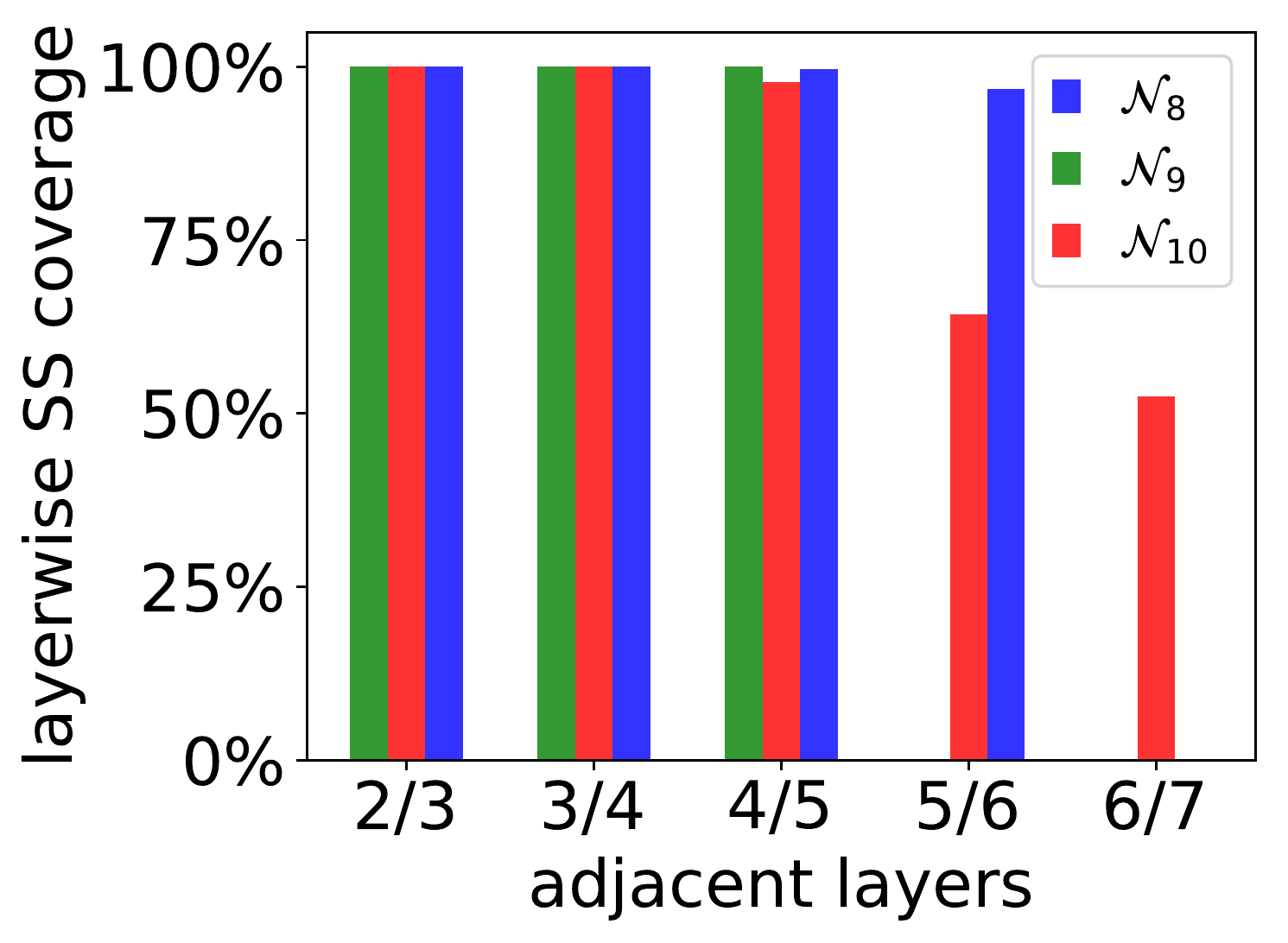}
    \label{fig:scc-layerwise-coverage}
  }\hfill
  &
  \subfloat[]{
    \includegraphics[width=0.45\linewidth]{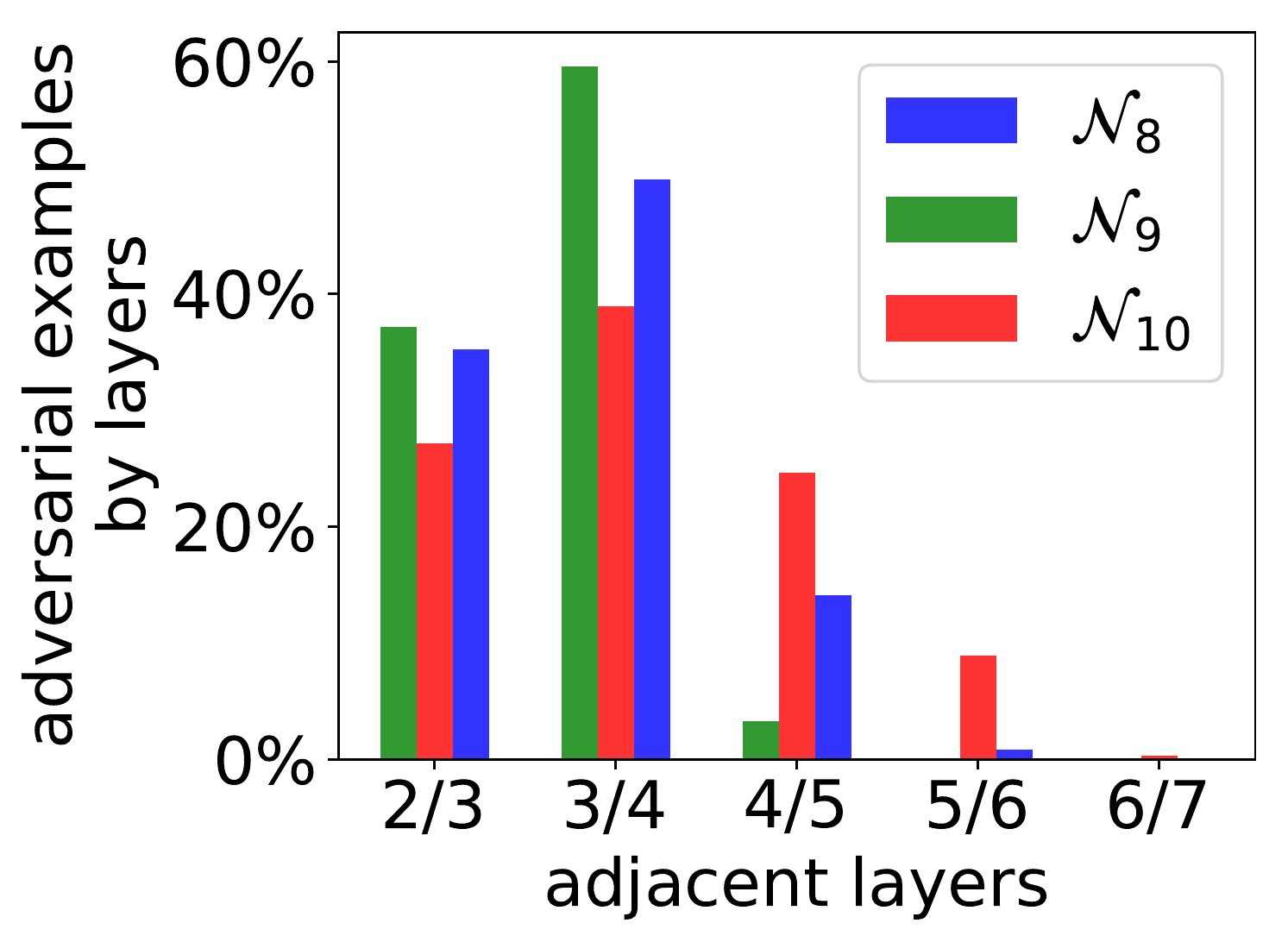}
    \label{fig:scc-layerwise-bugs}
  }
\end{tabular}
\caption{SS coverage by layer: (a) the coverage level per DNN layer; (b) the detected adversarial examples
at each layer with respect to the total amount of adversarial examples} 
\label{fig:ssc-layerwise}
\end{figure}
\paragraph{DNN Bug finding \ding{192}}

The testing results, as reported in Table \ref{tab:results}, are promising: (1) the test case generation algorithm
effectively achieves high coverage for all covering criteria, and (2) the covering methods 
are considered useful, supported by the fact that a significant
portion of adversarial examples are identified.
Figure~\ref{fig:ads} exhibits several adversarial examples found during the testing
with different distances.  
We note that, \emph{for neuron coverage \cite{PCYJ2017}, 
a high coverage can be easily achieved by selecting a few non-adversarial test cases that we generated.}

\paragraph{DNN safety analysis \ding{193}}

The coverage $M_f$ and adversarial example percentage $AE_f$ together provide
quantitative statistics to evaluate a DNN.  Generally speaking, given a test suite, a DNN with
a high coverage level $M_f$ and a low adversarial percentage $AE_f$ is considered robust.  
In addition, we can study the \emph{adversarial quality} by plotting a distance curve
to see how close the adversarial example is to the correct input.
Take a closer look into the results of SS coverage for the last three DNNs in
Table~\ref{tab:results}.  As illustrated in Figure~\ref{fig:ss-distance-map},
the horizontal axis measures the $L^{\infty}$ distance and the vertical axis reports the accumulated
percentage of adversarial examples that fall into this distance.  A~more robust DNN will have
its shape in the small distance end (towards~0) lower, as the reported
adversarial examples are relatively farther from their original correct
inputs.  Intuitively, this means that more effort needs to be made 
to fool a robust DNN from correct classification into mislabelling.

\paragraph{Layerwise behavior \ding{195}}

Our experiments show that different layers of a DNN exhibit different behaviors in testing.
Figure~\ref{fig:ssc-layerwise} reports the SS coverage results, collected in
adjacent layers.  In particular, Figure~\ref{fig:scc-layerwise-coverage}
gives the percentage of covered neuron pairs within individual adjacent
layers.  As we can see, when going deeper into the DNN, it can become harder
to cover of neuron pairs.  Under such circumstances, to improve the
coverage performance, the use of larger a $data\_set$ is needed when generating test
pairs.  Figure~\ref{fig:scc-layerwise-bugs} gives the percentage
of adversarial examples found at different layers (among the overall adversarial examples).  Interestingly, it seems
that most adversarial examples are found when testing the middle layers.

\begin{figure}[!htb]
\centering
\includegraphics[width=0.55\columnwidth]{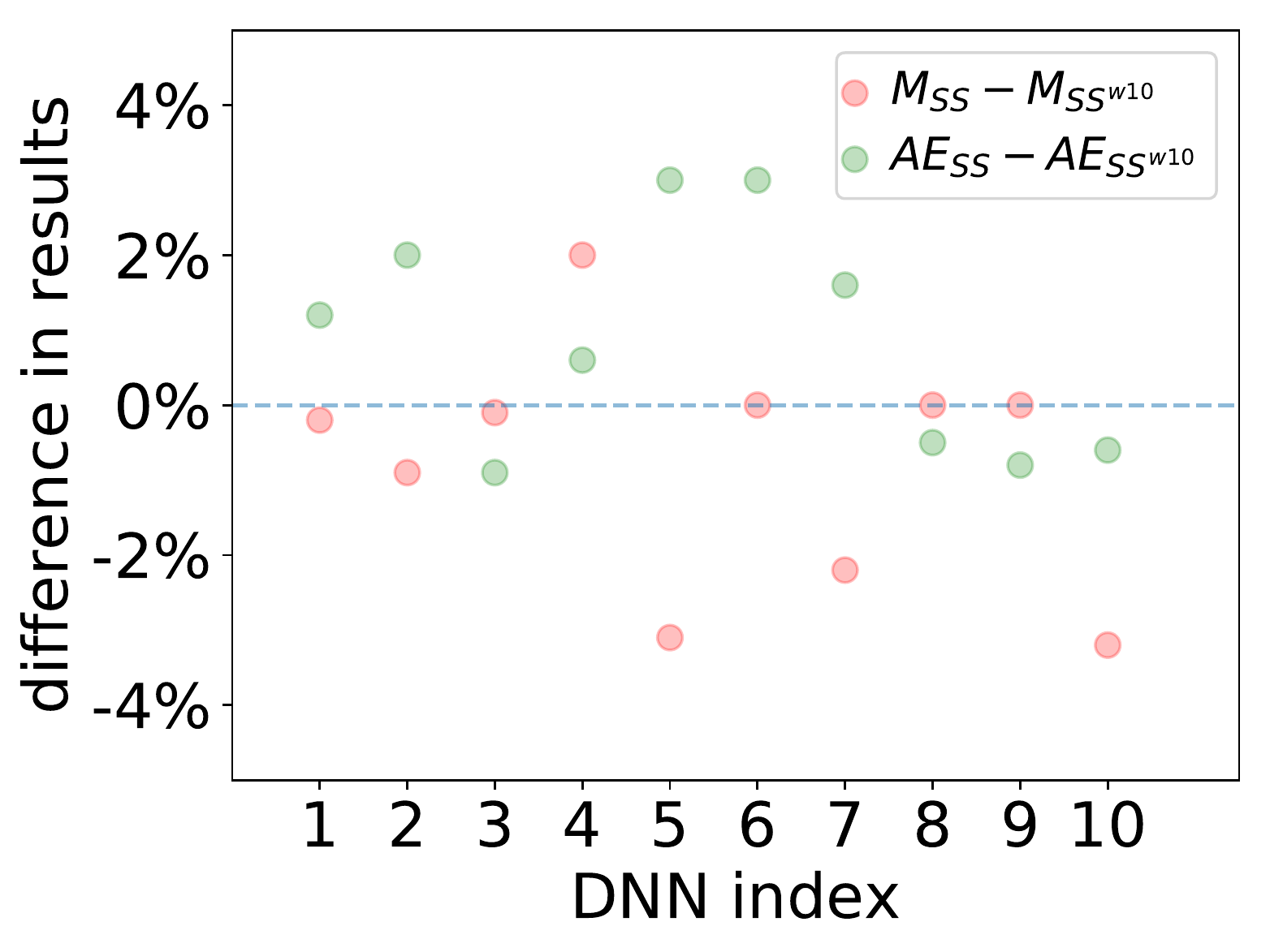}
\caption{$\covered{SS}{}{}$ vs. $\covered{SS}{w10}{}$. Results demonstrate that
the SS coverage and its top-weight simplification have similar coverage levels ($M_{SS}-M_{SS^{w10}}$)
and percentages of adversarial examples ($AE_{SS}-AE_{SS^{w10}}$)}
\label{fig:ss-top10}
\end{figure}
\paragraph{SS coverage with top weights \ding{194}}
For SS coverage criteria with neuron pairs, there are
totally $|O|$ test conditions for $O\subseteq \neuronpairs(\networks)$. 
We note that $|\neuronpairs(\networks)|= \sum_{k=2}^{K} s_k \cdot s_{k-1}$. 
To reduce the test suite size, we define 
$O$ as follows: $(\feature_{k,i},\feature_{k+1,j})\in O$ only when
the weight is one of the $\kappa$ largest among $\{|w_{k,i',j}|~|~i'\in
[1\dots s_{k}]\}$. The rationale is that condition neurons do not
equally affect their decision, and those with higher (absolute) weights are
likely to have a larger influence. 

Figure~\ref{fig:ss-top10} shows the difference, on coverage and adversarial example 
percentages, between SS coverage and its simplification with $\kappa=10$, denoted by $SS^{w10}$.  
In general, the two are comparable.  This is very
useful in practice, as the ``top weights'' simplification mitigates the size of the rsulting test suite, and it is thus able to behave as a faster pre-processing phase and
even provide an alternative with comparable results for SS coverage.

\paragraph{Cost of LP call \ding{194}}

Since LP encoding of the DNN (partial) activation pattern plays a key role
in the test generation, in this part we give details of the LP call cost,
even though LP is widely accepted as an efficient method. 
For every DNN, we select a set of neuron pairs, where each decision neuron
is at a different layer.  Then, we measure the number of variables and
constraints, and the time $t$ in seconds (averaged over 100
runs) spent on solving each LP call.  Results in Table
\ref{tab:lp-call} confirm that the LP model of a partial activation pattern
is indeed lightweight, and its complexity increases in a linear manner when
traversing into deeper layers of a DNN.
\begin{table}
\centering
\scalebox{1.}{
  \begin{tabular}{|l|ccc|ccc|ccc|} \hline
    \multirow{2}{*}{         } & \multicolumn{3}{c|}{$\networks_8$} &  \multicolumn{3}{c|}{$\networks_9$} & \multicolumn{3}{c|}{$\networks_{10}$}\\\cline{2-10}
                             & \#vars & $|\constraints|$ & $t$ & \#vars & $|\constraints|$ & $t$ & \#vars & $|\constraints|$ & $t$ \\\hline
    $L2\textnormal{-}3$             & 864 & 3294 & 0.58 & 873 & 3312 & 0.57 & 862 & 3290 & 0.49 \\
    $L3\textnormal{-}4$             & 926 & 3418 & 0.84 & 906 & 3378 & 0.61 & 917 & 3400 & 0.71 \\
    $L4\textnormal{-}5$             & 999 & 3564 & 0.87 & 968 & 3502 & 0.86 & 991 & 3548 & 0.75 \\
    $L5\textnormal{-}6$             & 1046& 3658 & 0.91 & --& --   &  --    & 1089& 3744 & 0.82 \\
    $L6\textnormal{-}7$             & -- & -- & -- & -- & -- & -- & 1164 & 3894 & 0.94 \\ 
    \hline
  \end{tabular}
}
\caption{Number of variables and constraints, and time cost of each LP call in test generation}
\label{tab:lp-call}
\end{table}

\subsection{CIFAR-10}

The CIFAR-10 dataset is a collection of 32x32 color images in ten kinds of objects. Different from the MNIST case,
we need to train a DNN with convolutional layers in order to handle the CIFAR-10 image classification problem.
Without loss of generality, the activation of a node in the convolutional layer 
is computed  by the activations of a subset of precedent nodes, and each node belongs to a \emph{feature map} in its layer. 
We apply the test case generation in Algorithm \ref{algo:aac-get-input-pair-search} for the SS coverage and  measure
the coverage results individually for decision features at each different layer. Overall, an SS coverage higher than 90\% is achieved
with a significant portion of adversarial examples.  
An interesting observation is as in Figure \ref{fig:cifar-adv-distribution} (\ding{195}), 
which shows that in this case
the causal changes of features at deeper layers are able to detect smaller perturbations of inputs that cause adversarial behaviours,
and this is likely to provide helpful feedback for developers to debug or tune the neural network parameters.
Selected adversarial examples are given in Figure \ref{fig:cifar-advs}. 

\begin{figure}[!htb]
\centering
\includegraphics[width=0.50\columnwidth]{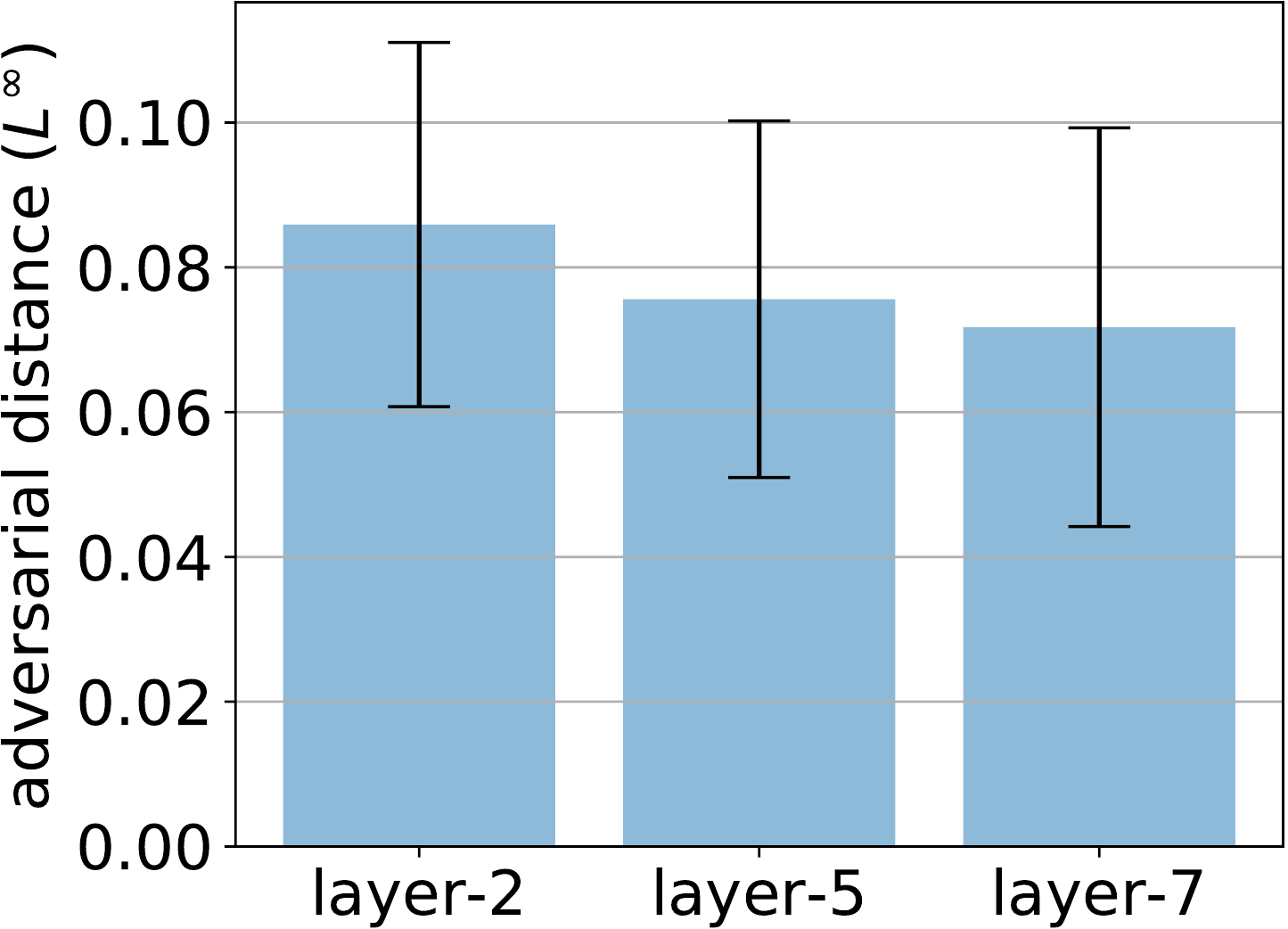}
\caption{The averaged adversarial distance for decision features at different layers}
\label{fig:cifar-adv-distribution}
\end{figure}

\begin{figure}[!htb]
\captionsetup{justification=centering}
\centering
\begin{tabular}{cc}
  \subfloat[bird $\rightarrow$ airplane]{
    \includegraphics[width=0.125\columnwidth]{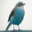}
    \includegraphics[width=0.125\columnwidth]{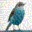}
    \label{fig:ad3}
  }
  &
  \subfloat[airplane $\rightarrow$ cat]{
    \includegraphics[width=0.125\columnwidth]{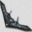}
    \includegraphics[width=0.125\columnwidth]{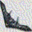}
    \label{fig:ad4}
  }
\end{tabular}
\caption{Selected adversarial examples for CIFAR-10}
\label{fig:cifar-advs}
\end{figure}

\subsection{ImageNet}

We applied our methods to VGG16~\cite{vgg16}, 
a large-scale DNN trained on the ImageNet dataset. 
The heuristic search Algorithm~\ref{algo:aac-get-input-pair-search} is called to generate test cases. 
We consider each
decision feature as a single set of neurons. 
While we can use feature extraction methods such as SIFT \cite{SIFT} to obtain condition features, in our experiments we consider each
condition feature as an arbitrary set of neurons for better exploration of the testing method.  
In particular, a size parameter $\omega$ is defined for the
experiments such that a feature $\feature_{k,i}$ is required to have its size $\leq\omega\cdot s_k$. Recall that 
 $s_k$ is the number of neurons in layer $k$.

\paragraph{Different feature sizes~\ding{192}~\ding{193}~\ding{195}}
We apply SS coverage on 2,000 randomly sampled feature pairs 
with 
$\omega\in\{0.1\%, 0.5\%, 1.0\%\}$. The covering method shows its effectiveness by returning
a test suite in which  $10.5\%$, $13.6\%$ and $14.6\%$ are
adversarial examples. 
We report the adversarial examples' average distance and standard deviation
in Figure~\ref{fig:ssc-errorbar}.  The results confirm that there is a
relation between the feature pairs and the input perturbation.  Among the
generated adversarial examples, a more fine-grained feature is able to
capture smaller perturbations
than a coarse one.

\begin{figure}[!htb]
\centering
\includegraphics[width=0.55\columnwidth]{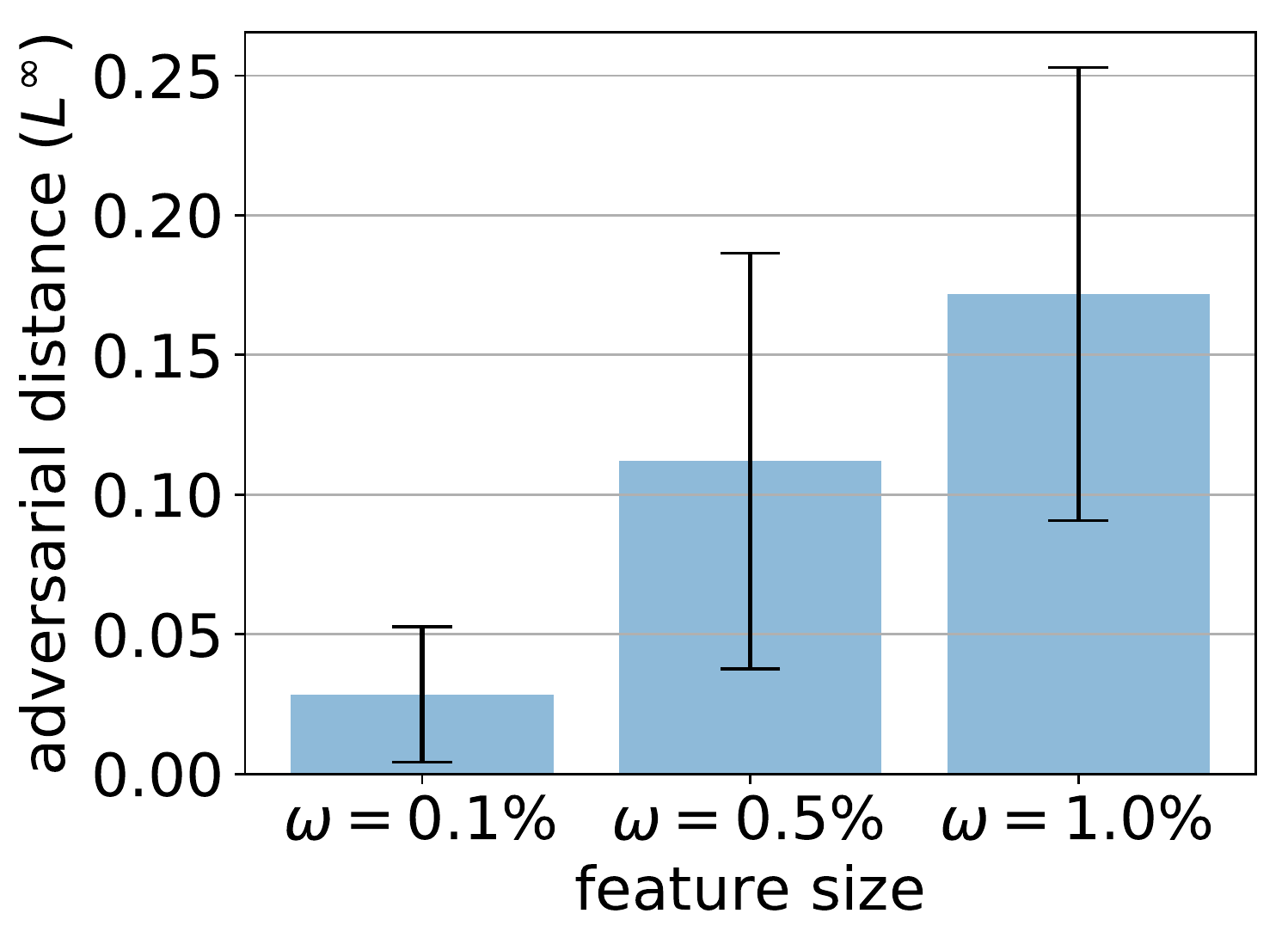}
\caption{Adversarial distance with different feature sizes: a smaller distance corresponds to more subtle
adversarial examples}
\label{fig:ssc-errorbar}
\end{figure}

Results in Figure~\ref{fig:ssc-errorbar} are measured with $L^{\infty}$-norm that 
corresponds to the maximum changes 
to a pixel. We observed that, though the change of each pixel is very small, for every adversarial example
a large portion (around $50\%$) of pixels are changed. A typical adversarial example image 
is given in Figure~\ref{fig:adv-traffic-light}. Overall, the detected adversarial examples are considered of high quality.

\begin{figure}[!htb]
\centering
\includegraphics[width=.75\columnwidth]{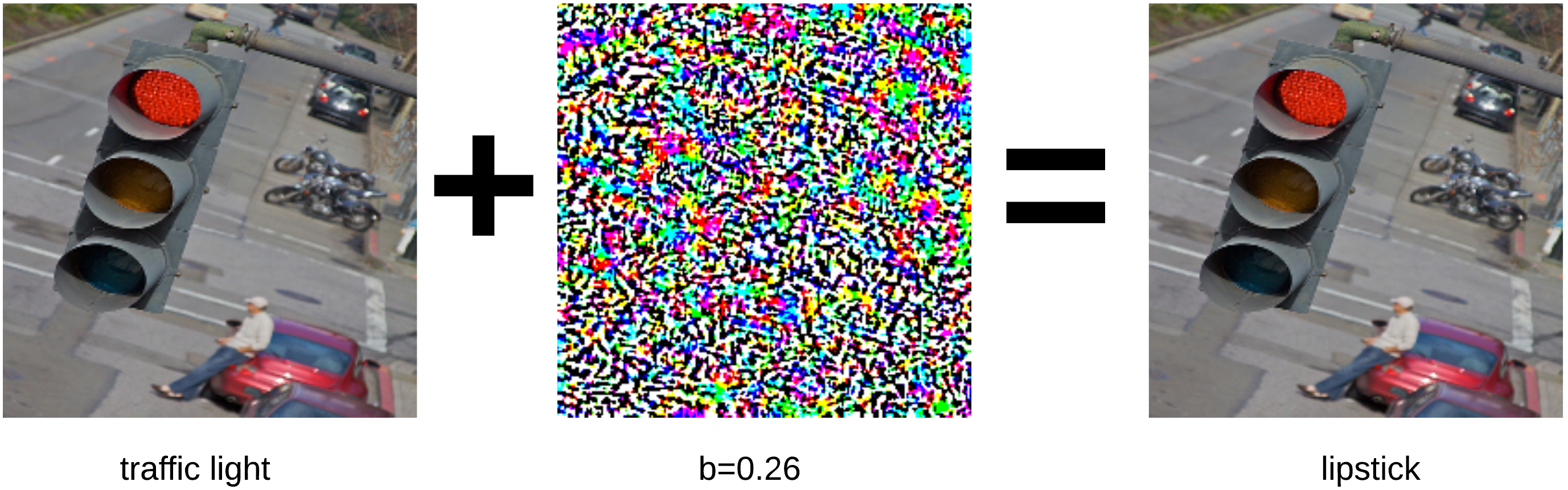}
\caption{An adversarial example (``lipstick'') for the original traffic light input} 
\label{fig:adv-traffic-light}
\end{figure}

\paragraph{SV with neuron boundary coverage~\ding{192}~\ding{193}}
\vspace{-0.45cm}
As shown in Section \ref{sec:comparison}, our covering methods are stronger than
neuron boundary coverage. In fact, neuron boundary is a special case of SV coverage,
when the value function 
of the decision feature is designed to make
the activation exceed the specified boundary value. We also validated this relation
in the empirical manner, similarly to the experiments
above, by generating a test suite using SV with neuron boundary coverage. We noticed that
accessing boundary activation values is likely to request bigger changes to be made in DNNs.
We set the feature size using $\omega=10\%$ and obtain a test suite with
$22.7\%$ adversarial examples. 
However, the distance of these adversarial examples, with average $L^{\infty}$-norm distance 3.49 and standard deviation 3.88, is much greater than those
for the SS coverage, as in Figure \ref{fig:ssc-errorbar}.
\commentout{
\begin{figure}[!htb]
\centering
\includegraphics[width=0.5\columnwidth]{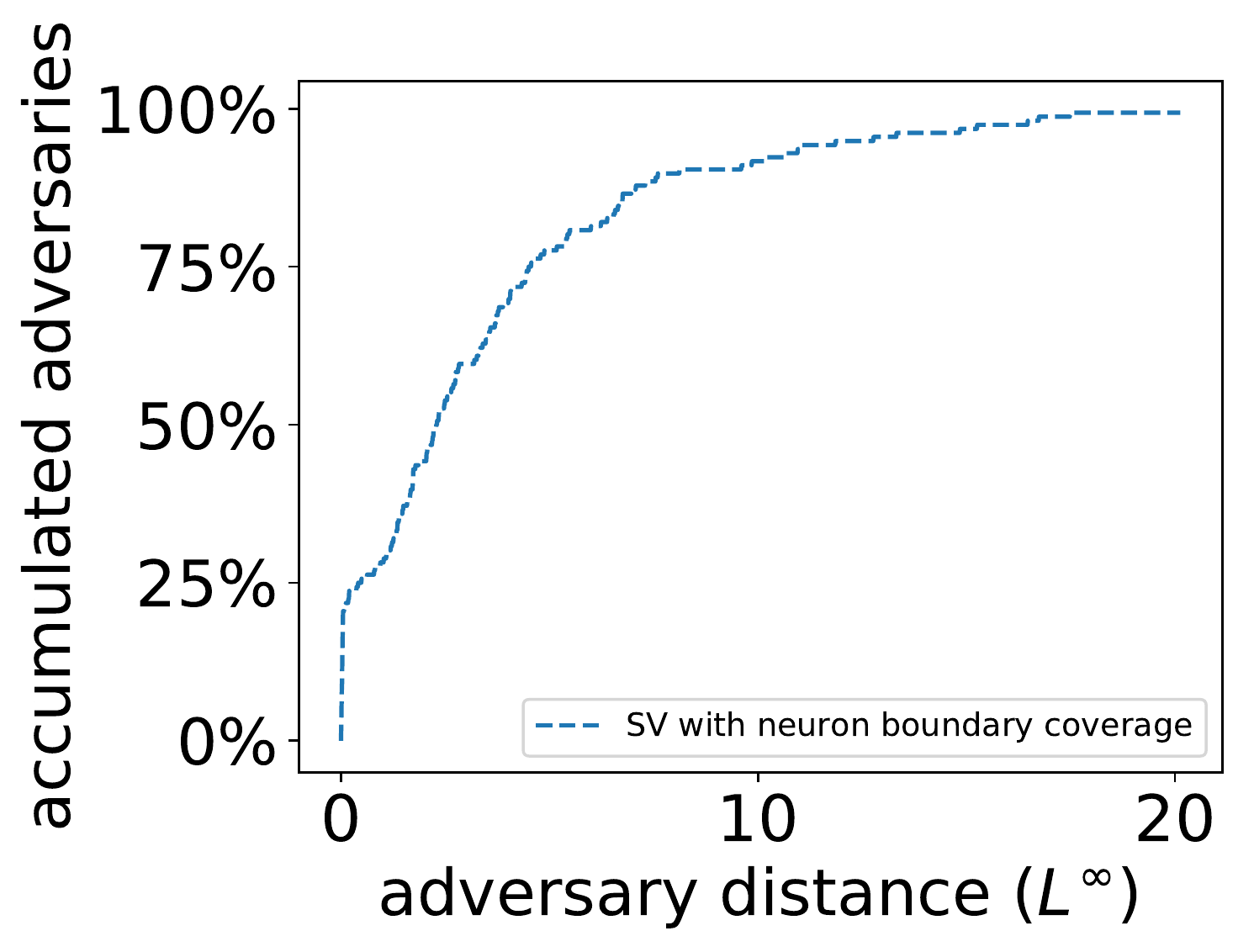}
\caption{The adversarial example curve when using SV with neuron boundary coverage}
\label{fig:sv-curve}
\end{figure}
}

\section{Related Work} 
\label{sec:related}

In the following, we briefly discuss existing
techniques looking to validate safety properties of DNNs.

\paragraph{Generation of Adversarial Examples for DNNs}

Most existing work, e.g.,~\cite{szegedy2014intriguing,FGSM,PMJFCS2015,CW2016} 
applies various heuristic algorithms, generally using search
algorithms based on gradient descent or evolutionary techniques. 
These approaches may be able to find adversarial examples efficiently, but are \emph{not able to
provide any guarantee} (akin to verification) or \emph{any certain level of
confidence} (akin to testing) about the nonexistence of adversarial examples
when the algorithm fails to find one.

\paragraph{Testing of DNNs}

At present, there are only a few proposals for structural DNN test coverage criteria. 
In~\cite{PCYJ2017}, {neuron coverage} is proposed to cover each neuron's binary activation
statuses. It is applied in \cite{tian2017deeptest} to guide the testing of DNN-driven autonomous
cars. Extensions of {neuron coverage} are made in
\cite{ma2018deepgauge}, which include a set of test criteria to check the 
corner values of a neuron's activation and the activation levels of
a subset of neurons in the same layer. However, criteria in \cite{PCYJ2017,ma2018deepgauge}
simply ignore the key causal relationship in a DNN.
Odena and Goodfellow \cite{odena2018tensorfuzz}
apply the approximate nearest neighbors algorithm to guide their tests generation,
but it is not clear, in a DNN, what the maximum number of nearest neighbors are.
As shown in \cite{lan2018design}, quantitative DNN coverage criteria 
can be applied to the design and certification of automotive systems with deep learning components.
%
%
%
%

In~\cite{WHK2018}, the input space is discretised with
hyper-rectangles, and then one test case is generated for each
hyper-rectangle.  
The resulting safety coverage is a strong criterion, but the generation
of a test suite can be very expensive.
Whilst in~\cite{huang-atva18}, coverage is enforced
to finite partitions of the input space, relying on predefined sets of
application-specific scenario attributes. 
The "boxing clever" technique in~\cite{box-clever} focuses on the distribution of training data and
divides the input domain into a series of representative boxes.
In \cite{kim2018guiding}, the difference between test dataset and training dataset
is measured by quantifying the difference between DNNs' activation patterns.

Some traditional test case generation techniques such as concolic
testing~\cite{sun2018concolic,sun2019concolic}, symbolic
execution~\cite{gopinath2018symbolic} and fuzzing~\cite{odena2018tensorfuzz,xie2018coverage}
have been recently extended to DNNs.  Mutation testing has similarly been investigated
in~\cite{wang2018adversarial,wang2018detecting,deepmutation,cheng2018manifesting,shenmunn}.
And metamorphic testing~\cite{ding2017validating,dwarakanath2018identifying,deeproad} has been identified as a suitable test
oracle for the robustness problem.  The combinatorial method is explored to
reduce the testing space for DNNs in~\cite{ma2018combinatorial}. 
Multi-implementation testing is applied to $k$-Nearest
Neighbor (kNN) and Naive Bayes supervised learning algorithms in~\cite{xie18a}. 
In~\cite{udeshi2018automated}, the adversarial inputs are treated as the
fairness problem via testing.

Tensorflow~\cite{tensorflow} is a popular library for developing deep
learning models, and Zhang et al~\cite{xiong-issta18} studied a collection
of 175 bugs in Tensorflow programs.  A testing framework is developed
in~\cite{xie18b} for learning based malware detection applications in
Android.  Autonomous driving is the primary application domain for
assessments of DNN testing techniques~\cite{yaghoubi-hscc,dreossi2017systematic,tuncali2018simulation}.

\paragraph{Automated Verification of DNNs}

The safety problem of a DNN can be reduced into a constraint solving problem~\cite{dreossi2019formalization}.
SMT~\cite{HKWW2017,PT2010,katz2017reluplex,tuncali2018reasoning}, MILP
\cite{bunel2017piecewise,LM2017,CNR2017,xiang2017output,dutta2018output} and SAT
\cite{NKPSW2017,narodytska2018formal} solutions 
have already been considered. In \cite{verisig}, the DNN is transformed into  into an equivalent hybrid
system. These approaches typically only work with small networks
with a few hundred hidden neurons, and approximation techniques \cite{mirman2018differentiable,dreossi2018compositional,wong2018provable,reluval,gehr2018ai,Dutta2019,sun2019hscc} can be
applied to improve the efficiency. 
Another thread of work~\cite{RHK2018,ruan2018global,WWRHK2018} based on global optimisation is promising to work with larger networks.

%

\section{Conclusions}
\label{sec:concl}

We have proposed a set of novel test criteria for DNNs. 
Our experiments on various datasets and test case generation methods
show promising results, indicating the feasibility and effectiveness of
the proposed test criteria. 
%
%
The test coverage 
metrics developed within this paper provide a method to
obtain evidence towards adversarial robustness, which is envisaged to
contribute to safety cases.  The 
metrics
%
%
are also expected to provide additional insights for domain experts when
they are considering the adequacy of a particular dataset for use in an
application.  

%
%
%


\subsubsection*{Acknowledgements}

This document is an overview of UK MOD (part) sponsored research and is released for informational purposes only. The contents of this document should not be interpreted as representing the views of the UK MOD, nor should it be assumed that they reflect any current or future UK MOD policy. The information contained in this document cannot supersede any statutory or contractual requirements or liabilities and is offered without prejudice or commitment.

Content includes material subject to \textcopyright~ Crown copyright (2018), Dstl. This material is licensed under the terms of the Open Government Licence except where otherwise stated. To view this licence, visit \url{http://www.nationalarchives.gov.uk/doc/open-government-licence/version/3} or write to the Information Policy Team, The National Archives, Kew, London TW9 4DU, or email: \email{psi@nationalarchives.gsi.gov.uk}.

\bibliographystyle{unsrt}
\bibliography{all}

\newpage
\appendix

\section*{Appendix}

This section gives the proofs for relations given in Section~\ref{sec:comparison}. 


\commentout{

\begin{theorem}
$M_{\covered{N}{}{}} \preceq M_{\covered{SS}{}{}}$, 
$M_{\covered{N}{}{}} \preceq M_{\covered{VS}{\valuefunction}{}}$, 
$M_{\covered{N}{}{}} \preceq M_{\covered{MN}{m}{}}$, 
$M_{\covered{MN}{m}{}} \preceq M_{\covered{SV}{\valuefunction}{}}$, 
$M_{\covered{MN}{m}{}} \preceq M_{\covered{VV}{\valuefunction_1,\valuefunction_2}{}}$, 
$M_{\covered{NB}{}{}} \preceq M_{\covered{SV}{\valuefunction}{}}$, 
$M_{\covered{NB}{}{}} \preceq M_{\covered{VV}{\valuefunction_1,\valuefunction_2}{}}$, $M_{\covered{TN}{m}{}} \preceq M_{\covered{VV}{\valuefunction_1,\valuefunction_2}{}}$, $M_{\covered{TN}{m}{}} \preceq M_{\covered{VS}{\valuefunction}{}}$, where for some relations, functions $\valuefunction$, $\valuefunction_1$ and $\valuefunction_2$ need to be instantiated into suitable  functions. 
\end{theorem}

}

\begin{lemma}\label{lemma:nss}
$M_{\covered{N}{}{}} \preceq M_{\covered{SS}{}{}}$.
\end{lemma}
\begin{proof}
Note that, for every hidden  node $n_{k,j}\in {\cal H}(\networks)$, there exists a feature pair $ (\{n_{k-1,i}\},\{n_{k,j}\})\in \neuronpairs(\networks)$ for any $1\leq i\leq s_{k-1}$. Then, by Definition~\ref{def:ssc}, we have $sc(\{n_{k,j}\},x_1,x_2)$, which by Definition~\ref{def:sc} means that $\mathit{sign}(n_{k,j},x_1) \neq
\mathit{sign}(n_{k,j},x_2)$. That is, either $\mathit{sign}(n_{k,j},x_1)=+1$ or $\mathit{sign}(n_{k,j},x_2)=+1$. Therefore, if $n_{k,j}$ is not covered in a test suite $\testsuites_1$ for neuron coverage, none of the pairs $ (\{n_{k-1,i}\},\{n_{k,j}\})$ for $1\leq i\leq s_{k-1}$ is covered in a test suite $\testsuites_2$ for SS coverage. 
\end{proof}

\begin{lemma}
$M_{\covered{N}{}{}} \preceq M_{\covered{VS}{\valuefunction}{}}$.
\end{lemma}
\begin{proof}
Follow a similar argument with Lemma~\ref{lemma:nss}. 
\end{proof}

\begin{lemma}
$M_{\covered{N}{}{}} \preceq M_{\covered{MN}{m}{}}$, when the interval $[v_{k,i}^l,v_{k,i}^u]$ is non-trivial, i.e., not $[0,0]$, for all nodes $n_{k,i}\in {\cal H}$.
\end{lemma}
\begin{proof}
Because $m$ subsections are all in $[v_{k,i}^l,v_{k,i}^u]$ and  $[v_{k,i}^l,v_{k,i}^u]$ is non-trivial, we have that the neuron coverage of a node $n_{k,i}$ is satisfied whenever any subsection of  $[v_{k,i}^l,v_{k,i}^u]$ is filled. 
\end{proof}

We remark that the condition about the non-trivial intervals are reasonable. First of all, in practice, all the DNNs we work with satisfy this condition. Second, if a node always have value $0$ for all the training samples then such a node can be seen as redundant.

\begin{lemma}\label{lemma:mnsv}
$M_{\covered{MN}{m}{}} \preceq M_{\covered{SV}{\valuefunction}{}}$ for a suitable function $\valuefunction$. 
\end{lemma}
\begin{proof}
For every hidden node, the upper bound $v_{k,i}^u$ and lower bounds $v_{k,i}^l$ are obtained from the training samples. Therefore, for any given subsection of $[v_{k,i}^l,v_{k,i}^u]$, we know its exact interval, say $[v_{k,i}^{l,j},v_{k,i}^{u,j}]$ for the $j$-th subsection. Then we can use the function $g$ to express that the value $v_{k,i}[x_2]$ is in $[v_{k,i}^{l,j},v_{k,i}^{u,j}]$. The value of $v_{k,i}[x_1]$ does not matter. Therefore, if  the subsections $[v_{k,i}^{l,j},v_{k,i}^{u,j}]$ is not covered, we know that the  feature pairs $(\{n_{k-1,j}\},\{n_{k,i}\})$ for $1\leq j\leq s_{k-1}$ have not been covered by the SV coverage under the function $g$. 
\end{proof}

\begin{lemma}
$M_{\covered{MN}{m}{}} \preceq M_{\covered{VV}{\valuefunction_1,\valuefunction_2}{}}$  for a suitable function $\valuefunction_2$.
\end{lemma}
\begin{proof}
Follow a similar argument with Lemma~\ref{lemma:mnsv}. 
\end{proof}

\begin{lemma}\label{lemma:nbsv}
$M_{\covered{NB}{}{}} \preceq M_{\covered{SV}{\valuefunction}{}}$  for a suitable function $\valuefunction$.
\end{lemma}
\begin{proof}
Follow a similar argument with Lemma~\ref{lemma:mnsv}, except that in this case the function $g$ is used to express that $v_{k,i}$ is greater than $v_{k,i}^{u}$.
\end{proof}

\begin{lemma}
$M_{\covered{NB}{}{}} \preceq M_{\covered{VV}{\valuefunction_1,\valuefunction_2}{}}$  for a suitable function $\valuefunction_2$. 
\end{lemma}
\begin{proof}
Follow a similar argument as that of Lemma~\ref{lemma:nbsv}.
\end{proof}

\begin{lemma}\label{lemma:tnvv}
$M_{\covered{TN}{m}{}} \preceq M_{\covered{VV}{\valuefunction_1,\valuefunction_2}{}}$  for a suitable function $\valuefunction_1$. 
\end{lemma}
\begin{proof}
We can work with feature pairs $(\{n_{k,i},n_{k+1,j}\})$ and  use $\valuefunction_1$ to express that $v_{k,i}[x_2]$ is greater than at least $s_k-m$ values in the set $\{v_{k,l}~|~l\in [1..s_k]\}$. The value of $v_{k,i}[x_1]$ does not matter. Therefore, whenever the node $n_{k,i}$ is not top-$m$ neuron covered then the pair is not VV covered under the functions $\valuefunction_1$ and $\valuefunction_2$ for any $\valuefunction_2$.  
\end{proof}

\begin{lemma}
 $M_{\covered{TN}{m}{}} \preceq M_{\covered{VS}{\valuefunction}{}}$  for a suitable function $\valuefunction_1$. 
\end{lemma}
\begin{proof}
Follow the similar argument with that of Lemma \ref{lemma:tnvv}. 
\end{proof}

We have the following conclusion stating the relationship between safety coverage and ours. 


\begin{theorem}
$M_{\covered{SS}{}{}} \preceq M_{\covered{S}{}{}}$. 
\end{theorem}
\begin{proof}
Note that, safety coverage exhaustively enumerates all possible activation patterns. Therefore, we have $M_{\covered{SS}{}{}} \preceq M_{\covered{S}{}{}}$ since the former only explore a subset of the activation patterns. 
\end{proof}

\end{document}